\documentclass[11pt, letterpaper]{article}

\usepackage[letterpaper, left=1in, right=1in, top=1in, bottom=1in]{geometry}

\usepackage{subcaption}

\usepackage{quoting}
\usepackage{csquotes}

\usepackage{algorithm}
\usepackage{algpseudocode}
\usepackage{setspace}

\usepackage{multicol}

\usepackage{dylan2}

\usepackage[dvipsnames]{xcolor}
\usepackage[colorlinks=true, linkcolor=blue, citecolor=blue]{hyperref}
\usepackage{mdframed}

\usepackage{natbib}
\bibpunct{(}{)}{;}{a}{,}{,}

\usepackage{amsthm}

\theoremstyle{definition}  

\newtheorem{lemma}{Lemma}

\newtheorem{corollary}{Corollary}
\newtheorem{proposition}{Proposition}

\newtheorem{model}{Model}
\theoremstyle{plain}

\newtheorem{example}{Example}
\newtheorem{theorem}{Theorem}
\newtheorem{definition}{Definition}

\xpatchcmd{\proof}{\itshape}{\normalfont\proofnameformat}{}{}
\newcommand{\proofnameformat}{\bfseries}

\newcommand{\extend}{\textsc{Extend}}
\newcommand{\Ot}{\wt{O}}

\newcommand{\avgdeg}{\Delta_{\mathrm{avg}}}

\newcommand{\mc}[1]{\mathcal{#1}}
\newcommand{\vs}{\crl{\pm{}1}^{V}}
\newcommand{\es}{\crl{\pm{}1}^{E}}

\newcommand{\Yh}{\widehat{Y}}
\newcommand{\Ys}{Y^{\star}}
\newcommand{\Yt}{\wt{Y}}

\newcommand{\pmo}{\crl{\pm{}1}}

\newcommand{\mincut}[1]{\mathrm{\sf{mincut}}\prn*{#1}}
\newcommand{\mincuts}[1]{\mathrm{\sf{mincut}}^{\star}\prn*{#1}}

\newcommand{\wt}[1]{\widetilde{#1}}

\newcommand{\Yhws}{\Yh^{\Ws}}

\newcommand{\tW}{\mc{W}}
\newcommand{\W}{\mc{W}}
\newcommand{\Ws}{W^{\star}}
\newcommand{\tF}{F}
\newcommand{\tree}{(\tW, \tF)}

\newcommand{\width}{\mathrm{\sf{wid}}}
\newcommand{\widths}{\mathrm{\sf{wid}}^{\star}}
\newcommand{\tw}{\mathrm{\mathsf{tw}}}
\renewcommand{\deg}{\mathrm{\mathsf{deg}}}

\newcommand{\cost}{\textrm{Cost}}

\newcommand{\Xe}{X^{(e)}}

\newcommand{\wh}[1]{\widehat{#1}}
\newcommand{\mainalg}{\textsc{TreeDecompositionDecoder}}
\newcommand{\treealg}{\textsc{TreeDecoder}}

\title{Inference in Sparse Graphs with Pairwise Measurements and Side Information}
\author{
Dylan J. Foster \thanks{Department of Computer Science, Cornell University. Supported in part by the NDSEG fellowship.}
\and
Daniel Reichman \thanks{Electrical Engineering and Computer Science, University of California, Berkeley}
\and
Karthik Sridharan \thanks{Department of Computer Science, Cornell University}
}
\date{}

\begin{document}

\maketitle

\begin{abstract}

We consider the statistical problem of recovering a hidden ``ground truth'' binary labeling for the vertices of a graph up to low Hamming error from noisy edge and vertex measurements.  We present new algorithms and a sharp finite-sample analysis for this problem on trees and sparse graphs with poor expansion properties such as hypergrids and ring lattices. Our method generalizes and improves over that of \cite{Globerson2015Hard}, who introduced the problem for two-dimensional grid lattices.

For trees we provide a simple, efficient,
algorithm that infers the ground truth with optimal Hamming error has optimal sample complexity and implies recovery results for \emph{all connected graphs}. Here, the presence of side information is critical to obtain a non-trivial recovery rate.
We then show how to adapt this algorithm to tree decompositions of edge-subgraphs of certain graph families such as lattices, resulting in optimal recovery error rates that can be obtained efficiently

The thrust of our analysis is to 1) use the tree decomposition along with edge measurements to produce a small class of viable vertex labelings and 2) apply an analysis influenced by statistical learning theory to show that we can infer the ground truth from this class using vertex measurements. We show the power of our method in several examples including hypergrids, ring lattices, and the Newman-Watts model for small world graphs. For two-dimensional grids, our results improve over \cite{Globerson2015Hard} by obtaining optimal recovery in the constant-height regime.
\end{abstract}



\section{Introduction}

Statistical inference over graphs and networks is a fundamental problem that has received extensive attention in recent years \citep{fortunato2010community,Krzakala2013Spectral,Abbe2014Decoding,hajek2014computational}. Typical inference problems involve noisy observations of discrete labels assigned to edges of a given network, and the goal is to infer a ``ground truth'' labeling of the vertices (perhaps up to the right sign)  that best explains these observations. Such problems occur in a wide range of disciplines including statistical physics, sociology, community detection, average case analysis, and graph partitioning. This inference problem is also related to machine learning tasks involving structured prediction that arise in computer vision, speech recognition and other applications such as natural language processing. Despite the intractability of maximum likelihood estimation, maximum a-posteriori estimation, and marginal inference for most network models in the worst case, it has been observed that approximate inference algorithms work surprisingly well in practice \citep{sontag2012tightening}, and recent work has focused on improving our theoretical understanding of this phenomenon \citep{Globerson2015Hard}. 

\cite{Globerson2015Hard} introduced a new inference model with the key feature that, in addition to observing noisy edge labels, one also observes noisy vertex labels. The main focus of the present paper is to further examine the extent to which the addition of noisy vertex observations improves the statistical aspects of approximate recovery.
Specifically, we analyze statistical recovery rates in \pref{mod:edge_vertex_measurement}.

As a concrete example, consider the problem of trying to recover opinions of individuals in social networks. Suppose that every individual in a social network can hold one of two opinions labeled by $-1$ or $+1$. We receive a measurement of whether neighbors in the network have the same opinion, but the value of each measurement is flipped with probability $p$. We further receive estimates of the opinion of each individual, perhaps using a classification model on their profile, but these estimates are corrupted with probability $q$. 

\begin{displayquote}
\begin{mdframed}
\begin{model}
\label{mod:edge_vertex_measurement}
We receive an undirected graph $G=(V,E)$ with $\abs{V}=n$, whose vertices are labeled according to an unknown ground truth $Y\in\vs$. We receive noisy edge measurements $X\in\es$, where $X_{uv}=Y_uY_v$ with probability $1-p$ and $X_{uv}=-Y_uY_v$ otherwise.  We receive ``side information'' vertex measurements $Z\in\vs$, where $Z_u=Y_u$ with probability $1-q$ and $Z_u=-Y_u$ otherwise. We assume $p<q<1/2$

Our goal is to produce a labeling $\Yh\in\vs$ such that with probability at least $1-o_{n}(1)$ the Hamming error $\sum_{v\in{}V}\ind\crl{\Yh_v\neq{}Y_v}$ is bounded by $O(f(p)n)$ where $\lim_{p\rightarrow 0}f(p)=0$.\end{model}
\end{mdframed}
\end{displayquote}

The reader may imagine the pairwise measurements as fairly accurate and the side information vertex estimates as fairly noisy (since the flip probability $q$ close to $1/2$). \pref{mod:edge_vertex_measurement} then translates to the problem of producing an estimate of the opinions of users in the social network which predicts the opinion of few users incorrectly.

A first step in studying recovery problems on graphs with noisy vertex observations was taken by \cite{Globerson2014Tight,Globerson2015Hard} who studied \pref{mod:edge_vertex_measurement} on square grid lattices. They proved that the statistical complexity of the problem is essentially determined by the number of cuts with cutset of size $k$, where $k$ ranges over nonnegative integers.
This observation, together with a clever use of planar duality, enabled them to determine the optimal Hamming error for the square grid. 

As in \cite{Globerson2014Tight,Globerson2015Hard} we focus on finding a labeling of low Hamming error (as opposed to \emph{exact recovery}, where one seeks to find the error probability that with which we recover all labels correctly). \cite{chen2016community} have recently considered exact recovery for edges in this setting for sparse graphs such as grid and rings. They consider the case where there are \emph{multiple} i.i.d observations of edge labels. In contrast, we focus on the case where there is a single (noisy) observation for each edge, on  side information, and on partial recovery\footnote{We refer the reader to \pref{app:related_work} for further discussion of related models.}.

The availability of vertex observations changes the statistical nature of the problem and --- as we will show --- enables nontrivial partial recovery rates in all sparsity regimes. For example, for the $n$-vertex path, it is not difficult \citep{Globerson2014Tight} to show that when there are only noisy edge observations any algorithm will fail to find the correct labeling (up to sign) of $\Omega(n)$ edges.
In contrast, we show that when noisy vertex observations are available, one can obtain a labeling whose expected Hamming error is at most $\Ot(pn)$ for any $p$.

Related community detection models such as the well known Stochastic Block Model (SBM) and Censored Block Model (CBM) consider the case where one wishes to detect two communities based on noisy edge observations. Namely, in these models only noisy edges observations are provided and one wishes to recover the correct labeling of vertices up to sign.
Block model literature has focused on graphs which have good expansion properties such as complete graphs, random graphs, and spectral expanders. By including side information, our model allows for nontrivial recovery rates and efficient algorithms for graphs with ``small" separators such as trees, thin grids, and ring lattices. Studying recovery problems in such ``non-expanding" graphs is of interest as many graphs arising in applications such as social networks \citep{flaxman2007expansion} have poor expansion. 

\paragraph{Challenges and Results}

The key challenge in designing algorithms for \pref{mod:edge_vertex_measurement} is understanding statistical performance: Even for graphs such as trees in which the optimal estimator (the marginalized estimator) can be computed efficiently, it is unclear what Hamming error rate this estimator obtains. Our approach is to tackle this statistical challenge directly; we obtain efficient algorithms as a corollary.

Our first observation is that the optimal Hamming error for trees is  $\wt{\Theta}(pn)$ provided $q$ is bounded away from $1/2$\footnote{The assumption on $q$ is necessary, as when $q$ approaches $1/2$
it is proven in \cite{Globerson2015Hard} that an error of $\Omega(n)$ is unavoidable for certain trees.}. This is obtained by an efficient message passing algorithm. We then (efficiently) extend our algorithm for trees to more general graphs using a tree decompositions of (edge)-subgraphs. Our main observation is that if we are given an algorithm that obtains a non-trivial error rate for inference in each constant-sized component of a tree decomposition, we can lift this algorithm to obtain a non-trivial error rate for the entire graph by leveraging side information. 

This approach has the advantage that it applies to non-planar graphs such as high dimensional grids; it is not clear how to apply the machinery of
\cite{Globerson2015Hard} to such graphs because planar duality no longer applies. Our decomposition-based approach also enables us to obtain optimal error bounds for ring lattices and thin grids which do not have the so-called weak expansion property that is necessary for the analysis in \cite{Globerson2015Hard}. See \pref{sec:examples} for an extensive discussion of concrete graph families we consider and the error rates we achieve. 

\subsection{Preliminaries}
We work with an undirected graph $G=(V,E)$, with $\abs{V}=n$ and $\abs{E}=m$. For $W\subseteq{}V$, we let $G(W)$ be the induced subgraph and $E(W)$ be the edge set of the induced subgraph. Let $N(v)$ be the neighborhood of a vertex $v$. When it is not clear from context we will use $N_{G}(v)$ to denote neighborhood with respect to a specific graph $G$. Likewise, for $S\subseteq{}V$ we use $\delta_{G}(S)$ to denote its cut-set (edges with one endpoint in $S$) with respect to $G$. For a directed graph, we let $\delta_{+}(v)$ denote the outgoing neighbors and $\delta_{-}(v)$ denote the incoming neighbors of $v$. For a subset $W\subseteq{}V$ we let $N_{G}(W) = \bigcup_{v\in{}W}N_G(v)$. We let $\deg(G)$ denote the maximum degree and $\avgdeg$ the average degree.

\paragraph{Parameter range}
We treat $q=1/2-\eps$ as constant unless otherwise specified. Furthermore, we shall assume throughout that $p \geq w(1/n)$, so the expected number of edge errors is super-constant. We use $\wt{O}$ to suppress $\log(n)$, $\log(1/p)$, and $1/\eps$ factors. We use the phrase ``with high probability'' to refer to events that occur with probability at most $1-o_{n}(1)$. 

In the appendix (\pref{thm:large_degree}) we show  that if the minimum degree of the graph is $\Omega(\log{}n)$ there is a trivial strategy that achieves arbitrarily small Hamming error. We therefore restrict to $\deg(G)$ constant, as this is representative of the most interesting parameter regime.


\section{Inference for Trees}
\label{sec:trees}

In this section we show how to efficiently and optimally perform inference in \pref{mod:edge_vertex_measurement} when the graph $G$ is a tree. As a starting point, note that the expected number of edges $(u,v)$ of the tree with $X_{uv}$ flipped is $p(n-1)$.  In fact, using a simple Chernoff bound, one can see that with high probability at most $2pn + \Ot(1) $ edges are flipped. This implies that for the ground truth $Y$, $\sum_{(u,v) \in E} \ind\crl{Y_u \ne X_{u,v} Y_v} \le 2 p n + \Ot(1)$ with high probability over sampling of the edge labels. Hence to estimate ground truth, it is sufficient to search over labelings $\Yh$ that satisfy the inequality
{\small
\begin{align}\label{eq:constree}
\sum_{(u,v) \in E} \ind\crl{\Yh_u \ne X_{u,v} \Yh_v} \le 2 p n + \Ot(1).
\end{align}}We choose the estimator that is most correlated with the vertex observations $Z$ subject to the aformentioned inequality. That is, we take $\Yh$ to be the solution to\footnote{See appendix for constants.}
{\small
\begin{align}
\begin{aligned}
& \textrm{minimize} &&\sum_{v \in V} \ind\crl{\Yh_v \ne Z_v}  \\ &\textrm{subject to } &&\sum_{(u,v) \in E} \ind\crl{\Yh_u \ne X_{u,v} \Yh_v} \le 2 p n + \Ot(1).\end{aligned}\label{eq:treeopt}
\end{align}}This optimization problem can be solved efficiently --- $O(\ceil{pn}^{2}n\deg(G))$ time for general trees and $O(\ceil{pn}n)$ time for stars and path graphs --- with message passing. The full algorithm is stated in \pref{app:algorithms}.

On the statistical side we use results from statistical learning theory to show that the Hamming error of $\Yh$ obtained above is with high probability bounded by $\Ot(p n)$. To move to the statistical learning setting (see \pref{app:statistical_learning} for an overview) we first define a ``hypothesis class'' $\F\defeq\crl{Y'\in\vs\mid{}\sum_{(u,v) \in E} \ind\crl{Y'_u \ne X_{u,v} Y'_v} \le 2 p n + \Ot(1) }$; note that this is precisely the set of $Y'$ satisfying \pref{eq:constree}. The critical observation here is that for any $\Yh$ the Hamming error (with respect to the ground truth) is proportional to the \emph{excess risk} in the statistical learning setting over $Z$ with class $\F$:
{\small
\begin{align}
\label{eq:excess_risk}
&\sum_{v \in V} \ind\crl{\Yh_v \ne Y_v}\\
&= \frac{1}{1-2q}\brk*{
\sum_{v \in V} \Pr_{Z}\crl{\Yh_v \ne Z_v} - \min_{Y'\in\F}\sum_{v \in V}\Pr_{Z}\crl{Y'_v \ne Z_v}
}.\notag
\end{align}}

Combining \pref{eq:excess_risk} with a so-called \emph{fast rate} from statistical learning theory (\pref{cor:fast_rate_well_specified}) implies that if we take $\Yh$ to be the \emph{empirical risk minimizer} over $\F$ given $Z$, which is in fact the solution to \pref{eq:treeopt}, then we have $\sum_{v\in{}V}\ind\crl{\Yh_v\neq{}Y_v}\leq{}O(\log(\abs{\F}/\delta)/\eps^2)$ with probability at least $1-\delta$. Connectivity of $G$ then implies $\abs{\F}\approx (\frac{e}{p})^{2pn+\Ot(1)}$, giving the final $\Ot(pn)$ rate.  \pref{theorem:tree_decoding} makes this result precise:

\begin{theorem}[Inference in Trees]
\label{theorem:tree_decoding}
Let $\Yh$ be the solution to \pref{eq:treeopt}. Then with probability at least $1-\delta$,
{\small
\begin{align}
\sum_{v\in{}V}\ind\crl*{\Yh_v\neq{}Y_v} &\leq{} \frac{1}{\eps^{2}}(2pn + 2\log(2/\delta) + 1)\log(2e/p\delta)\notag\\
&=\Ot(pn).\label{eq:tree_hamming_error}
\end{align}
}
\end{theorem}
We emphasize that side information is critical in this result. For trees --- in particular, the path graph --- no estimator can achieve below $\Omega(n)$ hamming error unless $p=O(1/n)$ \citep{Globerson2014Tight}.


\section{Inference for General Graphs}
\label{sec:inference}

\subsection{Upper Bound: Inference with Tree Decompositions} \label{sec:upper}

Our main algorithm, \mainalg{} (\pref{alg:decoder}) produces estimators for  \pref{mod:edge_vertex_measurement} for graphs $G$ that admit a \emph{tree decomposition} in the sense of Robertson and Seymour (\cite{robertson1986graph}). Recall that a tree decomposition for a graph $G=(V,E)$ is new graph $T=(\tW,F)$ in which each node in $\tW$ corresponds to a subset of nodes in the original graph $G$. The edge set $F$ forms a tree over $\tW$ and must satisfy a property known as \emph{coherence}, which guarantees that the connectivity structure of $T$ captures that of $G$. The approach of \mainalg{} is to use the edge observations $X$ to produce a \emph{local} estimator for each component of the tree decomposition $T$, then use the vertex observations $Z$ to combine the many local estimators into a single \emph{global} estimator.

Tree decompositions have found extensive use in algorithm design and machine learning primarily for computational reasons: These objects allow one to lift algorithmic techniques that are only feasible computationally on constant-sized graphs, such as brute force enumeration, into algorithms that run efficiently on graphs of all sizes. It is interesting to note that our algorithm obeys this principle, but for \emph{statistical performance} in addition to computational performance: We are able to lift an analysis technique that is only tight for constant-sized graphs, the union bound, into an analysis that is tight for arbitrarily large graphs from families such as grids. However, as our analysis for trees shows, this approach is only made possible by the side information $Z$.

The \emph{width} $\width(T)$ of a tree decomposition $T$ is the size of the largest component in $T$, minus one (by convention). To place a guarantee on the performance of \mainalg{}, both statistically and computationally, it is critical that the width be at most logarithmic in $n$. At first glance this condition may seem restrictive there are graphs of interests such as grids for which the \emph{treewidth} $\tw(G)$ --- the smallest treewidth of \emph{any} tree decomposition --- is of order $\sqrt{n}$.
 For such graphs, our approach is to choose a subset $E' \subseteq E$ of edges to probe so that the graph $G' = (V,E')$ has small treewidth. For all of the graphs we consider this approach obtains optimal sample complexity in spite of discarding information.

Having found a decomposition of small treewidth for $G'$ we apply the following algorithm. For each component of this decomposition, we compute the maximum likelihood estimator for the labels in this component given the edge measurements $X$. This is done by brute-force enumeration over vertex labels, which can be done efficiently because we require small treewidth.

For a given component, there will be two estimators that match the edges in that component equally well due to sign ambiguity. The remaining problem is to select a set of signs --- one for each component --- so that the local estimators agree globally. For this task we leverage the side information $Z$. Our approach will mirror that of \pref{sec:trees}: To produce a global prediction $\Yh$ we solve a global optimization problem over the tree decomposition using dynamic programming, then analyze the statistical performance of $\Yh$ using statistical learning theory.

Informally, if there is some $\Delta$ such that we can show a $p^{\Delta}$ failure probability for estimating up to sign the vertex labels within each component of the tree decomposition, the prediction produces by \pref{alg:decoder} will attain a high probability $p^{\Delta}n$ Hamming error bound for the entire graph. For example, in \pref{sec:examples} we show a $p^2$ failure probability for estimating vertex labels in a grid of size $3 \times 2$, which through  \pref{alg:decoder} translates to a $O(p^2 n)$ rate with high probability on both $\sqrt{n} \times \sqrt{n}$ and $3 \times n/3$ grids.

\begin{definition}[\cite{cowell2006probabilistic}]
\label{def:tree_decomposition}
A tree $T = \tree$ is a tree decomposition for $G=(V,E)$ if it satisfies
\begin{enumerate}
\vspace{-.1in}
\item \label{def:tree_decomp:1} \textbf{Vertex Inclusion:} Each node in $v\in{}V$ belongs to at least one component $W\in\tW$.
\item \label{def:tree_decomp:2} \textbf{Edge Inclusion:} For each edge $(u,v)\in{}E$, there is some $W\in\tW$ containing both $u$ and $v$.
\item \label{def:tree_decomp:3} \textbf{Coherence:} Let $W_1, W_2, W_3\in\tW$ with $W_2$ on the path between $W_1$ and $W_3$ in $T$. Then if $v\in{}V$ belongs to $W_1$ and $W_3$, it also belongs to $W_2$.
\end{enumerate}
\vspace{-.1in}
We assume without loss of generality that $T$ is not redundant, i.e. there is no $(W,W')\in{}F$ with $W'\subseteq{}W$.
\end{definition}

The next definition concerns the subsets of the graph $G$ used in the local inference procedure within \pref{alg:decoder}. We allow the local maximum likelihood estimator for a component $W$ to consider a superset of nodes, $\extend(W)$, whose definition will be specialized to different classes of graphs.

\begin{definition}[Component Extension Function]
\label{def:extend}
For a given $W\in\tW$, the \emph{extended component} $\Ws\supseteq{}W$ denotes the result of $\textsc{Extend}(W)$.
\end{definition}
Choices we will use for the extension function include the identity $\textsc{Extend}(W)=W$ and the neighborhood of $W$ with respect to the probed graph:
\begin{equation}
\textsc{Extend}(W) = \prn*{\bigcup_{v\in{}W}N_{G'}(v)}\cup{}W.
\end{equation}
Concrete instantiations of $\extend$ are given in \pref{sec:examples}.

We define quantitative properties of the tree decomposition in \pref{tab:tree_decomposition_properties}. For a given property, the corresponding ($\star$) version will denote the analogue the arises in analyzing performance when using extended components. For simplicity, the reader may wish to imagine each ($\star$) property as the corresponding non-($\star$) property on their first read-through.

\begin{table*}[htbp]
\centering
  \caption{Tree decomposition properties.}
  \label{tab:tree_decomposition_properties}
    \begin{tabular}{|ll|}
    \hline
$\deg(T)=\max_{W\in\W}\abs{\crl*{(W,W')\in{}F}}$ &\\
 $\width(T)=\max_{W\in\tW}\abs{W}-1$ & $\widths(T)=\max_{W\in\tW}\abs{\Ws}-1$\\
 $\mc{W}(e)=\crl*{W\in\W{}\mid{}e\in{}E(W)}$ & $\mc{W}^{\star}(e)=\crl*{W\in\W{}\mid{}e\in{}E(\Ws)}$\\
 $\deg_E(T)=\max_{e\in E}\abs{\mc{W}(e)}$ & $\deg^{\star}_E(T)=\max_{e\in E}\abs{\mc{W}^{\star}(e)}$\\
$\mincut{W}=\min_{S\subset{}W, S\neq{}\emptyset}\abs{\delta_{G(W)}(S)}$ &
$\mincuts{W}=\min_{S\subset{}\Ws, S\cap{}W\neq{}\emptyset, \bar{S}\cap{}W\neq{} \emptyset}\abs{\delta_{G(W)}(S)}$\\
 \hline
    \end{tabular}
\end{table*}

\begin{definition}[Admissible Tree Decomposition]
\label{def:admissible}
Call a tree decomposition $T=(\W,F)$ \emph{admissible} if it satisfies the following properties:

\begin{itemize}
\item $\deg(T)$, $\deg^{\star}_E(T)$, $\max_{W\in\W}\abs{E(W^{\star})}$, and $\widths(T)$ are constant.
\item $G'(W^{\star})$ is connected for all $W\in\W$\footnote{Together with our other assumptions, this implies the \emph{connected treewidth} of $G'$ \citep{diestel2016connected} is constant.}.
\end{itemize}
\end{definition}
In the rest of this section, the $\Ot$ notation will hide all of the constant quantities from \pref{def:admissible}.

\begin{theorem}[Main Theorem]
\label{theorem:main_simple}
Let $\Yh$ be the labeling produced using \pref{alg:decoder} with an admissible tree decomposition.
Then, with high probability over the draw of $X$ and $Z$,
{\small
\begin{equation}
\label{eq:main_w}
\sum_{v\in{}V}\ind\crl*{\Yh_v\neq{}Y_v} \leq{} \Ot\prn*{\sum_{W\in\tW}p^{\ceil{\mincuts{W}/2}}}.
\end{equation}}In particular, let $\Delta$ be such that $\Delta\leq{}\mincuts{W}$ for all $W\in\tW$. Then, with high probability,
{\small
\begin{equation}
\label{eq:eq:main_delta}
\sum_{v\in{}V}\ind\crl*{\Yh_v\neq{}Y_v} \leq{} \Ot\prn*{p^{\ceil{\Delta/2}}n}.
\end{equation}}\pref{alg:decoder} runs in time $\Ot(\ceil{p^{\Delta/2}n}^2n)$ for general tree decompositions and time $\Ot(\ceil{p^{\Delta/2}n}n)$ when $T$ is a path graph.
\end{theorem}
\begin{algorithm}[h!]
\begin{spacing}{1.2}
\caption{\textsc{TreeDecompositionDecoder}}\label{alg:decoder}
\textbf{Parameters:} Graph $G=(V,E)$. Probed edges $E'\subseteq{}E$. Extension function $\textsc{Extend}$.\\~~~~~~~~~Tree decomposition $T=(\tW, F)$ for $(V,E')$. Failure probability $\delta > 0$.\\
\textbf{Input:} Edge measurements $X\in\es$. Vertex measurements $Z\in\vs$.
\algrenewcommand\algorithmicindent{.8em}
\begin{algorithmic}[1]
\Procedure{TreeDecompositionDecoder}{}
\Statex \textbf{\textcolor{black}{\textsc{\underline{Stage 1} }}}
\Statex \textbf{\textcolor{blue}{\texttt{/* Compute estimator for each tree decomposition component. */}}}
\For{$W\in\tW$}
\State $W^{\star}\gets{}\textsc{Extend}(W)$. 
\Statex \indent\indent\indent\indent{\small\textbf{\textcolor{blue}{\texttt{// See \pref{def:extend}. }}}}
\State $\Yt^{\Ws}\leftarrow{}$ {\small$\argmin\limits_{\Yt\in\pmo^{\Ws}}\sum_{uv\in{}E'(\Ws)}\ind\crl{\Yt_u\Yt_v\neq{}X_{uv}}$}.
\State Let $\Yhws$ be the restriction of $\Yt^{\Ws}$ to $W$.\label{line:estimator}
\EndFor
\Statex \textbf{\textcolor{black}{\textsc{\underline{Stage 2} }}}
\Statex \textbf{\textcolor{blue}{\texttt{/* Use component estimators to assign edge costs to tree decomposition. */}}}
\For{$W\in\tW$}
\State $\mathrm{Cost}_W[+1] \leftarrow \sum_{v \in W} \ind\crl{\Yh^{\Ws}_v \ne Z_v} $ \State 
$\mathrm{Cost}_W[-1] \leftarrow \sum_{v \in W} \ind\crl{- \Yh^{\Ws}_v \ne Z_v}$.
\EndFor
\For{$(W_1,W_2) \in F$ }\label{line:decoder_meta_edge}
\State Let $v\in{}W_1 \cap W_2$.
\State {$S(W_1,W_2) \leftarrow \Yh^{\Ws_1}_{v} \cdot \Yh^{\Ws_2}_{v}$}.
\EndFor
\Statex \textbf{\textcolor{blue}{\texttt{/* Run tree inference algorithm from \pref{sec:trees} over tree decomposition. */}}}
\State$\hat{s} \gets \treealg{}(T, \textrm{Cost},S, L_n)$.\label{line:s_hat}
\Statex \indent\indent\indent\indent{\small\textbf{\textcolor{blue}{\texttt{// See eq. \pref{eq:ln} for constant $L_n$. }}}}
\For{$v\in{}V$}
\State Choose arbitrary $W$ s.t. $v\in{}W$ \Statex \indent\indent and set $\Yh_{v}\leftarrow{}\hat{s}_{W}\Yh^{\Ws}_{v}$.
\EndFor\\
\Return $\Yh$.
\EndProcedure
\end{algorithmic}
\end{spacing}
\end{algorithm}

\subsection{Main theorem: Proof sketch}
Let us sketch the analysis of \pref{theorem:main_simple} in the simplest case, where $\textsc{Extend}(W)=W$ for all $W\in\tW$ and consequently all $(\star)$ properties are replaced with their non-$(\star)$ counterparts. We give a bound begin by bounding that probability that a single component-wise estimator $\Yh^{W}$ computed on \pref{line:estimator} of \pref{alg:decoder} fails to exactly recover the ground truth within its component.
\begin{definition}[Component Estimator]
The (edge) maximum likelihood estimator for $W$ is given by
{\small
\begin{equation}
\label{eq:component_estimator}
\Yh^{W}\defeq{}\argmin_{\Yh\in\pmo^{W}}\sum_{uv\in{}E'(W)}\ind\crl{\Yh_u\Yh_v\neq{}X_{uv}}.
\end{equation}}
$\Yh^W$ can be computed by enumeration over all labelings in time $2^{\abs{W}}$.
There are always two solutions to \pref{eq:component_estimator} due to sign ambiguity; we take one arbitrarily.
\end{definition}

\begin{proposition}[Error Probability for Component Estimator]
\label{prop:component_estimator}
{\small
\[
\Pr\left(\min_{s\pmo} \ind\crl{s\Yh^W \ne Y^W} >0\right) \le \Ot(p^{\lceil\mincut{W}/2\rceil})
\]}
\end{proposition}
\begin{proof}
Assume that both $\Yh^W$ and $-\Yh^W$ disagree with the ground truth or else we are done. Let $S$ be a maximal connected component of the set of vertices $v$ for which $\Yh^W_v\neq{}Y_v$. It must be the case that at least $\ceil{\abs{\delta(S)}/2}$ edges $(u,v)$ in $\delta(S)$ have $X_{uv}$ flipped from the ground truth, or else we could flip all the vertices in $S$ to get a new estimator that agrees with $X$ better than $\Yh$; this would be a contradiction since $\Yh$ minimizes $\sum_{uv\in{}E'(W)}\ind\crl{\Yh_u\Yh_v\neq{}X_{uv}}$. Applying a union bound, the failure probability is bounded by
{\small
\[
 \sum_{S\subseteq{}W:S\neq{}\emptyset, S\neq{}W}p^{\ceil{\abs{\delta(S)/2}}}\leq{} 2^{\abs{W}}p^{\ceil{\mincut{W}/2}}.
\]}

\end{proof}
\pref{prop:component_estimator} bounds the probability of failure for \emph{individual components},
 but does not immediately imply a bound on the total number of components that may fail for a given realization of $X$. If the components $\W$ did not overlap one could apply a Chernoff bound to establish such a result, as their predictions would be independent. Since components can in fact overlap their predictions are dependent, but using a sharper concentration inequality (from the \emph{entropy method} \citep{Boucheron2003Concentration}) we can show that --- so long as no edge appears in too many components --- an analogous concentration result holds and total number of components failures is close to the expected number with high probability.
  \begin{lemma}[Informal]
  \label{lem:conc_informal}
With high probability over the draw of $X$,
 {\small
 \begin{align}
 \label{eq:conc_simple}
 \min_{s\in\pmo^{\tW}}\sum_{W\in\tW}\ind\crl{s_W\Yh^W \ne Y^W} &\leq \Ot\prn*{\sum_{W\in\W}p^{\ceil{\mincut{W}/2}}}. \end{align}}
     \end{lemma}
 In light of \pref{eq:conc_simple}, consider the signing of the component-wise predictions $(\Yh^{W})$ that best matches the ground truth.
 {\small
\[
s^{\star} = \argmin_{s\in\pmo^{\tW}}\sum_{W\in\tW}\ind\crl{s_W\Yh^W \ne Y^W}.
\]}If we knew the value of $s^{\star}$ we could use it to produce a vertex prediction with a Hamming error bound matching \pref{eq:main_w}. Computing $s^{\star}$ is not possible because we do not have access to $Y$. We get the stated result by proceeding in a manner similar to the algorithm \pref{eq:treeopt} for the tree. We first define a class $\F\subseteq{}\pmo^{\tW}$ which has the property that 1) $s^{\star}\in\F$ with high probability and 2) $\abs{\F}\lesssim{} 2^{\Ot\prn*{\sum_{W\in\W}2^{\abs{W}}p^{\ceil{\mincut{W}/2}}}}$. Then we take the component labeling $\hat{s}$ is simply the element of $\F$ that is most correlated with the vertex observations $Z$: $\hat{s}=\argmin_{s\in\F}\sum_{W\in\tW}\sum_{v\in{}W}\ind\crl*{s_{W}\Yh^{W}_v\neq{}Z_v}$ (this is \pref{line:s_hat} of \pref{alg:decoder}). Finally, to produce the final prediction $\Yh_v$ for a given vertex $v$, we find $W\in\tW$ with $v\in{}W$ and take $\Yh_{v}=\hat{s}_{W}\cdot{}\Yh^{W}_v$. A generalization bound from statistical learning theory then implies that this predictor enjoys error at most $\Ot(\log\abs{\F}) = \Ot\prn*{\sum_{W\in\W}2^{\abs{W}}p^{\ceil{\mincut{W}/2}}}$, which establishes the main theorem.

\paragraph{Efficient implementation} Both the tree algorithm and \pref{alg:decoder} rely on solving a constrained optimization problem of the form \pref{eq:treeopt}. In \pref{app:algorithms} we show how to perform this procedure efficiently using a message passing scheme.

\subsection{Lower Bounds: General Tools}
In this section we state simple lower bound techniques for  \pref{mod:edge_vertex_measurement}.
Recall that we consider $q$ as a constant, and thus we are satisfied with lower bounds that coincide with our upper bounds up to polynomial dependence on $q$.

\begin{theorem}
\label{theorem:general_lb}
Assume $p<q$. Then any algorithm for  \pref{mod:edge_vertex_measurement} incurs expected hamming error $\Omega(\sum_{v\in{}V}p^{\ceil{\deg(v)/2}})$.
\end{theorem}

\begin{corollary}
Any algorithm for \pref{mod:edge_vertex_measurement} incurs expected hamming error $\Omega(p^{\avgdeg/2+1}n)$.
\end{corollary}

\begin{theorem}
\label{theorem:system_lb}
Let $\tW$ be a collection of disjoint constant-sized subsets of $V$. Then for all $p$ below some constant, any algorithm for \pref{mod:edge_vertex_measurement} incurs expected Hamming error $\Omega(\sum_{W\in\tW}p^{\ceil{\abs{\delta_{G}(W)}/2}})$.
\end{theorem}


\section{Concrete Results for Specific Graphs}
\label{sec:examples}
We now specialize the tools developed in the previous section to provide tight upper and lower bounds on recovery for concrete classes of graphs.

\subsection{Connected Graphs}
\begin{example}[Arbitrary graphs]
\label{ex:connected}
For any connected graph $G$, the following procedure attains an error rate of $\Ot(pn)$ with high probability: {\bf 1.} Find a spanning tree $T$ for $G$. {\bf 2.} Run the algorithm from \pref{sec:trees} on $T$.

\noindent This rate is sharp, in the sense that there are connected graphs --- in particular, all trees --- for which  $\Omega(pn)$ Hamming error is optimal. Furthermore, for all graphs one can attain an estimator whose Hamming error is bounded as $\Ot(p n + \#\mathrm{connected~components})$ by taking a spanning tree for each component. This bound is also sharp.
\end{example}
The next example shows that there are connected graphs beyond trees for which $\Omega(pn)$ Hamming error is unavoidable.
More generally, $\Omega(pn)$ Hamming error is unavoidable for any graph with a linear number of degree-$2$ vertices.

Looking at \pref{theorem:general_lb}, one might be tempted to guess that the correct rate for inference is determined entirely by the degree profile of a graph. This would imply, for instance, that for any $d$-regular graph the correct rate is $\Theta(p^{\ceil{d/2}}n)$. The next example --- via \pref{theorem:system_lb} --- shows that this is not the case.
\begin{example}
\label{ex:dregular_lb}
For any constant $d$, there exists a family of $d$-regular graphs on $n$ vertices for which no algorithm in \pref{mod:edge_vertex_measurement} attains lower than $\Omega(pn)$ Hamming error.
\end{example}
This construction for $d=3$ is illustrated in \pref{fig:3regular}. We note that this lower bound hides a term of order $q^{\Omega(d)}$,
 but for constant $q$ and $d$ it is indeed order $\Omega(pn)$.

\begin{figure}[h!]
\centering
\includegraphics[scale=.2]{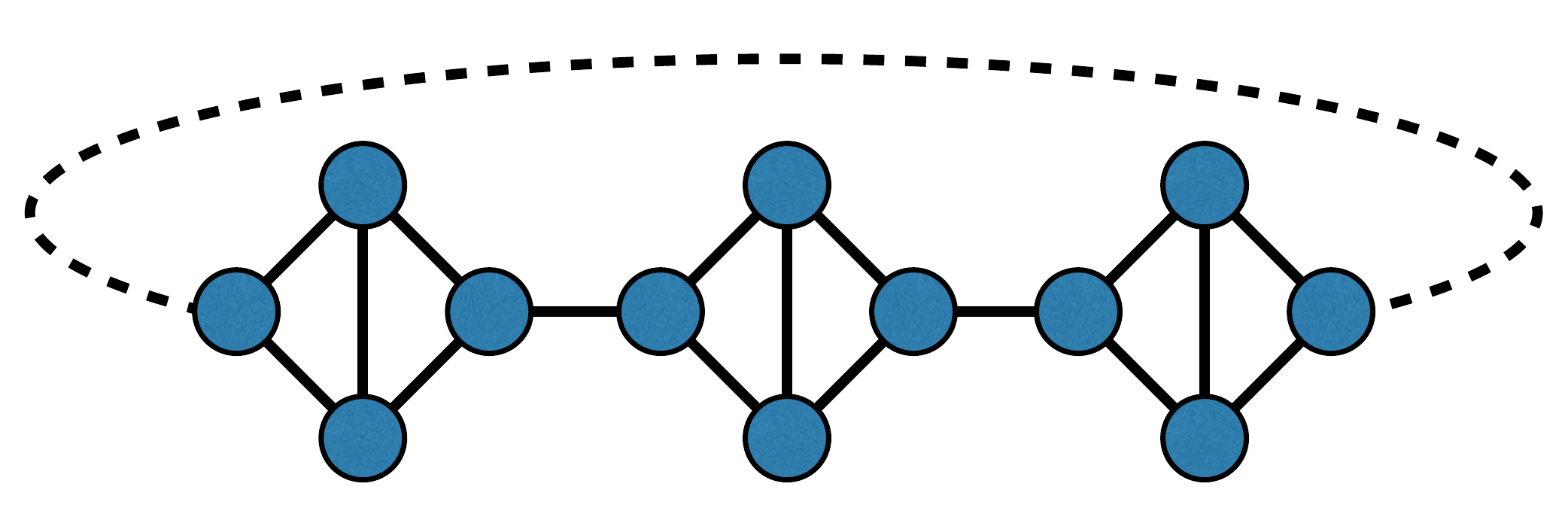}
\caption{$3$-regular graph for which $O(pn)$ error rate is optimal.}
\label{fig:3regular}
\end{figure}

\subsection{Grid Lattices}

\begin{figure}[h!]
\centering
\begin{subfigure}{.25\textwidth}
\centering
\includegraphics[scale=.35]{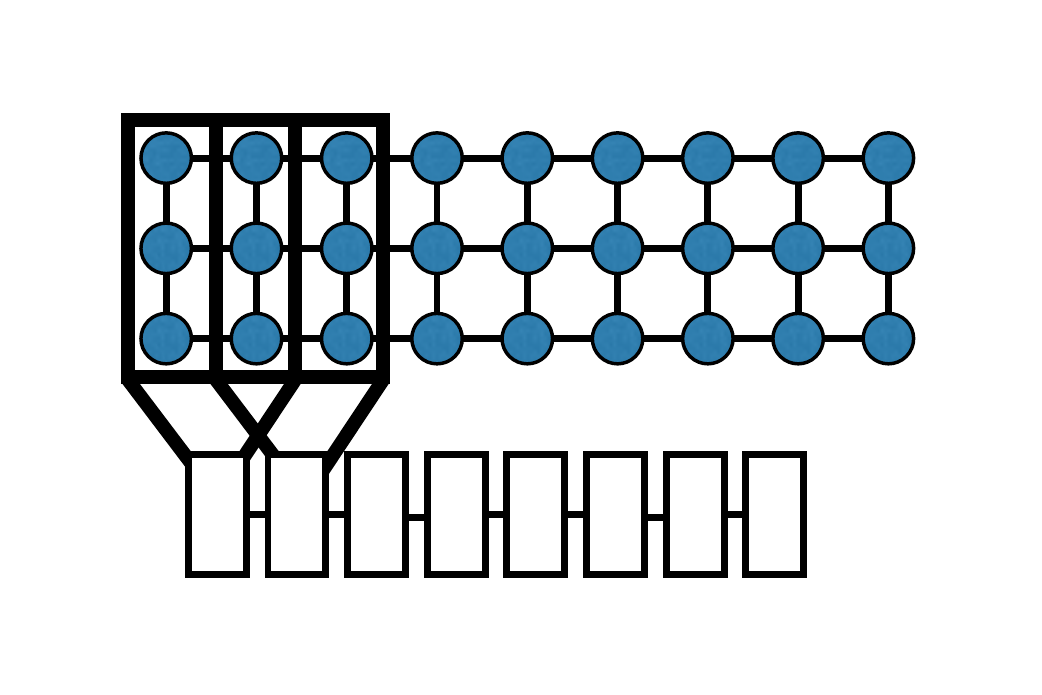}
\caption{Tree decomposition for $3\times{}n/3$ grid.}
\label{fig:const_height_decomp}
\end{subfigure}\hspace{.2\linewidth}\begin{subfigure}{.25\textwidth}
\centering
\includegraphics[scale=.3]{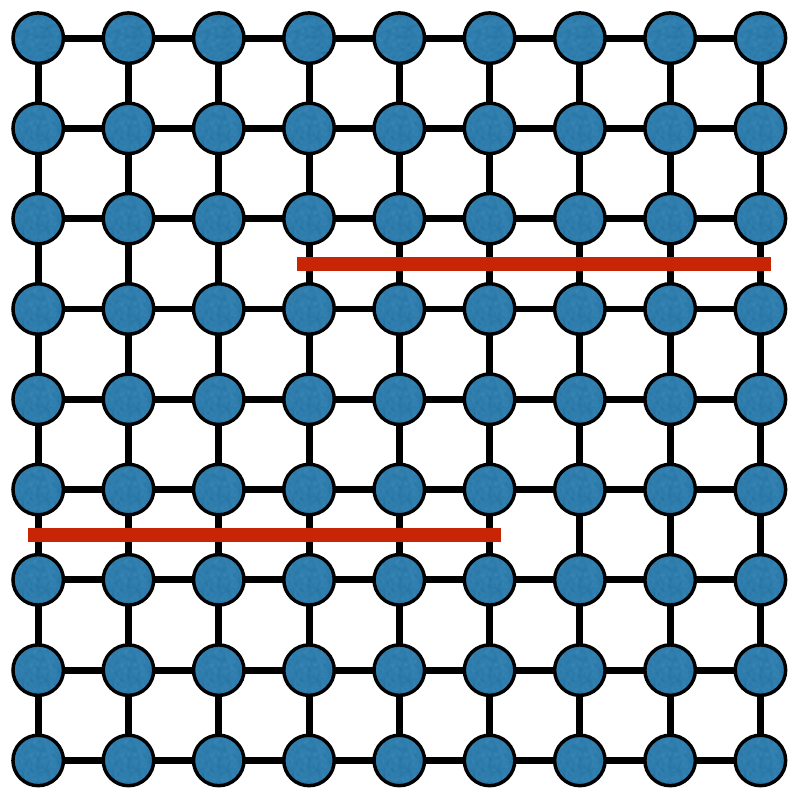}
\caption{$E'$ for $\sqrt{n}\times{}\sqrt{n}$ grid.\\~}
\label{fig:grid_cut}
\end{subfigure}
\caption{}
\label{fig:decomp_examples}
\end{figure}

In this section we illustrate how to use the tree-decomposition based algorithm, \pref{alg:decoder}, to obtain optimal rates for grid lattices.

\begin{example}[$2$-dimensional grid]
\label{ex:grid}
Let $G$ be a $2$-dimensional grid lattice of size $c \times n/c$ where $c \le \sqrt{n}$. For grid of height $c = 3$ (or above) using \pref{alg:decoder}, we obtain an estimator $\wh{Y}$ such that with high probability, the Hamming error is bounded as $O(p^2 n)$. This estimator runs in time $O(\ceil{p^2 n}n)$,
 By the degree profile argument (also given in \cite{Globerson2015Hard}), there is a matching lower bound of $\Omega(p^2 n)$. For a grid of height $c=1$ there is an obvious lower bound of $\Omega(pn)$ since this graph is a tree.

\end{example}
The estimator of \cite{Globerson2015Hard} can be shown to have expected Hamming error of $O(p^2 n)$ for the 2-dimensional grid with $c = \Omega(\log n)$. Our method works for \emph{constant height grids} ($c= O(1)$) and with high probability.

\pref{alg:decoder} of course requires a tree decomposition as input. The tree decomposition used to obtain \pref{ex:grid} for constant-height grids is illustrated in \pref{fig:const_height_decomp} for $c=3$: The grid is covered in overlapping $3\times{}2$ components, and these are connected as a path graph to form the tree decomposition.

The reader will observe that this tree decomposition has $\mincut{W}=2$, and so only implies a $O(pn)$ Hamming error bound through \pref{theorem:main_simple}. This rate falls short of the $O(p^{2}n)$ rate promised in the example; it is no better than the rate if $G$ were a tree. The problem is that within each $3\times{}2$ block, there are four ``corner'' nodes each with degree $2$. Indeed if either edge connected to a corner is flipped from the ground truth, which happens with probability $p$, this corner is effectively disconnected from the rest of $W$ in terms of information. To sidestep this issue, we define $\textsc{Extend}(W)=\bigcup_{v\in{}W}N(v)$. With this extension, we have $\mincuts{W}=3$ for all components except the endpoints, which implies the $O(p^{2}n)$ rate.

\paragraph{Probing Edges} We now illustrate how to extend the tree decomposition construction for constant-height grids to a construction for grids of arbitrary height.
Recall that \pref{alg:decoder} takes as input a subset $E'\subseteq{}E$ and a tree decomposition $T$ for $G'=(V,E')$. To see where using only a subset of edges can be helpful consider \pref{fig:const_height_decomp} and \pref{fig:grid_cut}. The $3\times{}n/3$ grid is ideal for our decoding approach because it can be covered in $3\times{}2$ blocks as in \pref{fig:const_height_decomp} and thus has treewidth at most 5. The $\sqrt{n}\times{}\sqrt{n}$ grid is more troublesome because it has treewidth $\sqrt{n}$, but we can arrive at $G'$ with constant treewidth by removing $\Theta(n)$ edges through the ``zig-zagging'' cut shown in \pref{fig:grid_cut}. Observe that once the marked edges in \pref{fig:grid_cut} are removed we can ``unroll'' the graph and apply a decomposition similar to \pref{fig:const_height_decomp}.

The tree decomposition construction we have outlined for two-dimensional grids readily lifts to higher dimension. This gives rise to the next example.

\begin{example}[Hypergrids and Hypertubes]
\label{ex:hypertube}
Consider a three-dimensional grid lattice of of length $n/c^2$, height $c$, and width $c$. If $c = n^{1/3}$ --- that is, we have a cube --- then \pref{alg:decoder} obtains Hamming error $\wt{O}(p^3 n)$ with high probability, which is optimal by \pref{theorem:general_lb}.

When $c$ is constant, however, the optimal rate is $\Omega(p^{2}n)$; this is also obtained by \pref{alg:decoder}. This contrasts the two-dimensional grid, where the optimal rate is the same for all $3\leq{}c\leq{}\sqrt{n}$.

\pref{alg:decoder} can be applied to any d-dimensional hypergrid of the form $c\times{}c\times{}\ldots{}n/c^{d-1}$ to achieve $O(p^{d}n)$ Hamming error when $c\approx{}n^{1/d}$. For constant $c$, the optimal rate is $\Theta(p^{\ceil{\frac{d+1}{2}}}n)$. More generally, the optimal rate interpolates between these extremes.
\end{example}

The next two examples briefly sketch how to apply tree decompositions to more lattices. Recall that the triangular lattice and hexagonal lattice are graphs whose drawings can be embedded in $\R^{2}$ to form regular triangular and hexagonal tilings, respectively.

\begin{example}[Triangular Lattice]
Consider a triangular lattice of height and width $\sqrt{n}$.
Let each component to be a vertex and its $6$ neighbors (except for the edges of the mesh), and choose these components such that the graph is covered completely. For a given component, let $W^{\star}$ to be the neighborhood of $W$. For this decomposition $\mincuts{W}$ is $6$ and consequently \pref{alg:decoder} achieves Hamming error $\wt{O}(p^{3}n)$. This rate is optimal because all vertices in the graph have degree $6$ besides those at the boundary, but the number of vertices on the boundary is sub-linear.
\end{example}
The triangular lattice example in particular shows that there exist graphs of average degree $3$ for which an error rate of $O(p^2 n)$ is achievable.

\begin{example}[Hexagonal Lattice]
Consider a $\sqrt{n}\times\sqrt{n}$ hexagonal lattice. Take each component $W$ to be a node $v$ and its neighbors, and choose the nodes $v$ so that the graph is covered. Choose $W^{\star}$ to be the neighborhood of the component $W$. The value of $\mincuts{W}$ for each component is $3$, leading to a Hamming error rate of $\wt{O}(p^{2}n)$. This rate is optimal because all vertices on the lattice except those at the boundary have degree $3$.
\end{example}

\subsection{Newman-Watts Model}

To define the Newman-Watts small world model \citep{newman1999renormalization}, we first define the regular ring lattice, which serves as the base graph for this model. The \emph{regular ring lattice} $R_{n,k}$ is a $2k$-regular graph on $n$ vertices defined as follows:  1) $V=\crl*{1,\ldots,n}$. 2) $E=\crl*{(i,j)\mid{} j\in\crl*{i+1,\ldots,i+k \;(\textrm{mod } n)}}$. \pref{theorem:general_lb} immediately implies that the best rate possible in this model is $\Omega(p^{k}n)$. Using \pref{alg:decoder} with an appropriate decomposition it is indeed possible to achieve this rate.
\begin{example}
\label{ex:ring_lattice}
The optimal Hamming rate for $R_{n,k}$ in \pref{mod:edge_vertex_measurement} is $\wt{\Theta}(p^{k}n)$. Moreover, this rate is achieved by an efficiently by \pref{alg:decoder} in time $O(\ceil{p^{k}n}n)$.
\end{example}
Note that for constant $k$, $R_{n,k}$ does not have the weak expansion property, and so the algorithm of \cite{Globerson2015Hard} does not apply.
We can now move on to the Newman-Watts model:
\begin{definition}[Newman-Watts Model]
To produce a sample from the Newman-Watts model $H_{n,k,\alpha}$, begin with $R_{n,k}$, then independently replace every non-edge with an edge with probability $\alpha/n$.
\end{definition}

For any constant $\alpha<1$, a constant fraction of the vertices in $R_{n,k}$ will be untouched in $H_{n,k,\alpha}$. Thus, the inference lower bound for \pref{ex:ring_lattice} still applies, meaning that the optimal rate is $O(p^{k}n)$. Algorithmically, this result can be obtained by discarding the new edges and using the same decomposition as in \pref{ex:ring_lattice}.
\begin{example}
\label{ex:newman_watts}
For any $\alpha<1$, the optimal Hamming rate for $H_{n,k,\alpha}$ in \pref{mod:edge_vertex_measurement} is $\wt{\Theta}(p^{k}n)$. Moreover, this rate is achieved in time $O(\ceil{p^{k}n}n)$ by \pref{alg:decoder} .
\end{example}


\section{Discussion}
\label{sec:discussion}

We considered \pref{mod:edge_vertex_measurement}, introduced in \cite{Globerson2015Hard}, for approximately inferring the ground truth labels for nodes of a graph based on noisy edge and vertex labels. We provide a general method to deal with arbitrary graphs that admit small width tree decompositions of (edge)-subgraphs. As a result, we recover the results in \cite{Globerson2015Hard} for grids, and are able to provide rates for graphs that do not satisfy the weak expansion property which is needed for the proof techniques in \cite{Globerson2015Hard}. Furthermore, in contrast to most existing work, we demonstrate that recovery tasks can be solved even on sparse ``nonexpanding" graphs such as trees and rings. 

There are several future directions suggested by this work. Currently, it is a nontrivial task to characterize the optimal error rate achievable for a given graph, and it is unclear how to extend our methods to families beyond lattices and graphs of small treewidth. Exploring connections between further graph parameters and achievable error rates is a compelling direction, as our understanding of optimal sample complexity remains quite limited. The challenge here entails both finding what can be done information-theoretically, as well as understanding what recovery rates can be obtained efficiently.


\bibliography{refs}

\appendix

\section{Further discussion of related work}

\label{app:related_work}

\paragraph{Computational Results for Markov Random Fields}

There is a long line of work on computational aspects of inference (e.g. MLE, MAP) in Markov Random Field models similar to \pref{mod:edge_vertex_measurement} \citep{veksler1999efficient, boykov2006graph, komodakis2007approximate, schraudolph2009efficient, chandrasekaran2012complexity}. To our knowledge none of these results shed light on the statistical recovery rates that are attainable for this setting ---  computationally efficiently or not.

\paragraph{Censored Block Model}

A recent line of research has studied recovery under the so-called
\emph{censored block model} (CBM). In CBM, vertices are labeled by $\pm 1$
and for every edge $uv$, the number $Y_uY_v$ is observed independently with probability $1-q$ (where $Y_u,Y_v$ are the labels of the vertices).
The goal is to find the true label $Y_uY_v$ of each edge $uv$ correctly with high probability
(based on the noisy observations).
For partial recovery in the censored block model we ask for a prediction whose correlation with the ground truth (up to sign) is constant strictly greater than $1/2$ as $n\to\infty$.
For the Erd\"{o}s-R\'{e}nyi 
random graph model, $G(n,\alpha/n)$ both the threshold (how large $\alpha$ needs to be in terms of $p$) for partial \cite{Saade2015Spectral}
and exact \cite{Abbe2014Decoding} recovery have been determined
Exact recovery is obtained through maximum likelihood estimation which is generally intractable.  The authors provide a polynomial time algorithm based on semidefinite programming that matches this threshold up to constant factors.

We observe that in our setting, due to the presence of side information, there is a simple and efficient algorithm that achieves exact recovery with high probability when the minimal degree is $\Omega(\log n)$: \pref{thm:large_degree}.
Such exact recovery algorithms are known for CBM model only under additional spectral expansion conditions \cite{Abbe2014Decoding}. 
\paragraph{Recovery from Pairwise Measurements}
\cite{Chen2014Information} provide conditions on exact recovery in a censored block model-like setting which, like our own, considers structured classes of graphs. Motivated by applications in computational biology and social networks analysis, \cite{chen2016community} have recently considered \emph{exact recovery} for edges in this setting. Like the present work, they consider sparse graphs with local structure such as grids and rings. Because their focus is \emph{exact} recover and their model does not have side information, their results mainly apply to graphs of logarithmic degree and our incomparable to our own results. For example, on the ring lattice $R_{n,k}$ in \pref{ex:ring_lattice} their exact recovery result requires $k=\Omega(\log(n))$, whereas our partial recover result concerns constant $k$.

\paragraph{Correlation Clustering}
Correlation clustering focuses on a combinatorial optimization problem closely related to the maximum likelihood estimation problem for our setting when we are only given edge labels. The main difference from our work is that the number of clusters is not predetermined. Most work on this setting has focused on obtaining approximation algorithms and has not considered any particular generative model for the weights (as in our case). An exception is \cite{Joachims2005Error}, which gives partial recovery results in a model similar to the one we consider, in which a ground truth partition is fixed and the observed edge labels correspond to some noisy notion of similarity. However, these authors focus on the case where $G$ is the complete graph.

\cite{Makarychev2015Correlation} consider correlation clustering where the model is a semi-random variant of the one we consider for the edge inference problem: Fix a graph $G=(V,E)$ and a vertex label $Y$. For each $uv\in{}E$, we observe $X_{uv}$ where $X_{uv}=Y_uY_v$ with probability $1-p$ and has its value in selected by an adversary otherwise. They do not consider side information, nor are they interested in concrete structured classes of graphs like grids.

\section{Omitted proofs}
\label{app:proofs}
\subsection{Proofs from \pref{sec:trees}}

\begin{proof}[Proof of \pref{theorem:tree_decoding}]
By the Bernstein inequality it holds that with probability at least $1-\delta/2$,
\[
\sum_{(u,v) \in E} \ind\crl{Y_u \ne X_{u,v} Y_v} \le 2 p n + 2\log(2/\delta).
\]
Thus, if we take $\F=\crl*{\Yh: \sum_{(u,v) \in E} \ind\crl{\Yh_u \ne X_{u,v} \Yh_v} \le 2 p n + 2\log(2/\delta)}$, then  $Y\in\F$ with probability at least $1-\delta/2$.

Fix $\Yh\in\vs$. We can verify by substitution that for each $v\in{}V$, 
\[
\ind\crl{\Yh_v\neq{}Y_v}=\frac{1}{1-2q}\brk*{\Pr_{Z}(\Yh_v\neq{}Z_v) - \Pr_{Z}(Y_v\neq{}Z_v)}.
\]
This implies that when $Y\in\F$ we have the following relation for Hamming error:
\[
\sum_{v\in{}V}\ind\crl*{\Yh_v\neq{}Y_v} = \frac{1}{1-2q}\brk*{\sum_{v\in{}V}\Pr(\Yh_v \ne Z_v) - \min_{Y'\in{}\F}\sum_{v\in{}V}\Pr\left(Y'_v \ne Z_v\right)}.
\]
\pref{cor:fast_rate_well_specified} now implies that if we take $\Yh=\argmin_{Y'\in\F}\sum_{v\in{}V}\ind\crl*{Y'_v\neq{}Z_v}$, which is precisely the solution to \pref{eq:treeopt}, then with probability at least $1-\delta/2$,
\[
\sum_{v\in{}V}\Pr\left(\Yh_v \ne Z_v\right) - \min_{Y'\in{}\F}\sum_{v\in{}V}\Pr\left(Y'_v \ne Z_v\right) \leq{}
\prn*{\frac{4}{3} + \frac{1}{\eps}}\log\prn*{\frac{2\abs{\F}}{\delta}}.
\]
Using that $\abs{\F}\leq{}\sum_{k=0}^{2pn+2\log(2/\delta)}{n\choose{}k}\leq{}(e/p)^{2pn + 2\log(2/\delta)}$ and $\eps\leq{}1/2$ we further have that the RHS is bounded as $\frac{2}{\eps}\log(2e/p\delta)(2pn + 2\log(2/\delta) + 1)$. Putting everything together (and recalling $1-2q=2\eps$), it holds that with probability at least $1-\delta$
\[
\sum_{v\in{}V}\ind\crl*{\Yh_v\neq{}Y_v} \leq{} \frac{1}{\eps^{2}}(2pn + 2\log(2/\delta) + 1)\log(2e/p\delta).
\]
\end{proof}

\subsection{Proofs from \pref{sec:inference}}
\begin{proof}[Proof of \pref{theorem:general_lb}]
The minimax value of the estimation problem is given by
\[
\min_{\Yh}\max_{Y}\En_{X,Z\mid{}Y}\sum_{v\in{}V}\ind\crl*{\Yh_v(X, Z)\neq{}Y_v}.
\]
We can move to the following lower bound by considering a game where each vertex predictor $\Yh_{v}$ is given access to the true labels $Y$ of all other vertices in $G$:
\[
\min_{\crl*{\Yh_v}_{v\in{}V}}\max_{Y}\En_{X,Z\mid{}Y}\sum_{v\in{}V}\ind\crl*{\Yh_v(X, Z, Y_{V\setminus\crl*{v}})\neq{}Y_v}.
\]
Under the new model, the minimax optimal predictor for a given node $v$ is given by the MAP predictor:
\[
\Yh_{v}=\argmin_{\Yh\in\pmo} \log\prn*{\frac{1-q}{q}}\ind\crl*{\Yh\neq{}Z_v} + \log\prn*{\frac{1-p}{p}}\sum_{u\in{}N_v}\ind\crl*{\Yh\neq{}Y_uX_{uv}}.
\]
When $p<q$, the minimax optimal estimator for $v$ takes the majority of the predictions suggested by its edges (that is, $Y_{u}\cdot{}X_{uv}$ for each neighbor $u$) and uses the vertex observation $Z_{v}$ to break ties.

When $\deg(v)$ is odd, the majority will be wrong if at least $\ceil{\deg(v)}$ of the edges in the neighbor of $v$ are flipped, and will be correct otherwise.

When $\deg(v)$ is even there are two cases: 1) Strictly more than $\ceil{\deg(v)}$ of the edges in $N(v)$ have been flipped, in which case the majority will be wrong. 2) Exactly half the edges are wrong, in which the optimal estimator will take the label $Z_v$ as its prediction, which will be wrong with probability $q$.
We thus have
\begin{align*}
\Pr(\Yh_v\neq{}Y_v) &= \sum_{k=\ceil{\deg(v)/2}}^{\deg(v)}{\deg(v)\choose{}k}p^{k}(1-p)^{\deg(v)-k}\\&\geq{} {\deg(v)\choose{}\ceil{\deg(v)/2}}p^{\ceil{\deg(v)/2}}(1-p)^{\deg(v)-k} \\
&\geq{} \prn*{\frac{\deg(v)}{\ceil{\deg(v)/2}}}^{\ceil{\deg(v)/2}}p^{\ceil{\deg(v)/2}}(1/2)^{\ceil{\deg(v)/2}} \\
&\geq{} \Omega(p^{\ceil{\deg(v)/2}}).
\end{align*}
In the last line we have used that we treat $\deg(v)$ as constant to suppress a weak dependence on it that arises when $\deg(v)$ is odd.
Putting everything together, we see that in expectation we have the bound
\[
\En\brk*{\sum_{v\in{}V}\ind\crl*{\Yh_{v}\neq{}Y_v}}  \geq{} \Omega\prn*{q\sum_{v\in{}V}p^{\ceil{\deg(v)/2}}}.
\]
\end{proof}

\begin{proof}[Proof of \pref{theorem:system_lb}]
Recall that the minimax value of the estimation problem is given by
\[
\min_{\Yh}\max_{Y}\En_{X,Z\mid{}Y}\sum_{v\in{}V}\ind\crl*{\Yh_v(X, Z)\neq{}Y_v}.
\]
As in the proof of \pref{theorem:general_lb}, we will move to a lower bound where predictors are given access to extra data. In this case, we consider a set of disjoint predictors $\crl*{\Yh^{W}}$, one for each component $W\in\tW$. We assume that $\Yh^{W}$ see the ground truth $Y_v$ for each vertex $v\notin{}W$, and further sees the product $Y_{uv}\defeq{}Y_uY_v$ for each edge $e\in{}E(W)$. Assuming $G(W)$ is connected (this clearly can only make the problem easier), the learner now only needs to infer one bit of information per component. The minimax value of the new game can be written as:
\[
\geq{}\min_{\crl*{\Yh^W}_{W\in\tW}}\max_{Y}\En_{X,Z\mid{}Y}\sum_{W\in\tW}\sum_{v\in{}W}\ind\crl*{\Yh^W_v(X, Z, Y_{V\setminus{}W}, \crl*{Y_{uv}\mid{}uv\in{}E(W)})\neq{}Y_v}.
\]
Because the learner only needs to infer a single bit per component, we have reduced to the setting of \pref{theorem:general_lb}, components in our setting as vertices in that setting (so $\deg(v)$ is replaced by $\delta_{G}(W)$). The only substantive difference is the following: In that lower bound, we required that $p<q$. For the new setting, we have that ``$q$'' is actually (pessimistically) $q^{\abs{W}}$, and so we require that $p<q^{\max_{W\in\tW}\abs{W}}$ for the bound to apply across all components. Using the final bound from \pref{theorem:general_lb}, we have
\[
\En\brk*{\sum_{v\in{}V}\ind\crl*{\Yh_{v}\neq{}Y_v}}  \geq{} \Omega\prn*{q^{\max_{W\in\tW}\abs{W}}\sum_{W\in\tW}p^{\ceil{\delta_{G}(W)/2}}}.
\]

\end{proof}

\subsection{Proofs from \pref{sec:examples}}

\begin{proof}[Proof of \pref{ex:connected}]

We will show that $\Omega(pn)$ Hamming error is optimal for all trees by establishing that all trees have constant fraction of vertices whose degree is at most two, then appealing to \pref{theorem:general_lb}.

Let $T$ be the tree under consideration. $T$ is bipartite. Let $(A,B)$ be the bipartition of $T$ into two disjoint independent sets. Suppose without loss of generality that $|A| \geq n/2$. 

If $a$ is the number of vertices in $A$ of degree at least $3$ and $a'=\abs{A}-a$, we have that $3a \leq n-1$, hence $a \leq (n-1)/3$. Therefore $a'\geq n/2-a\geq (n-1)/6$. Letting $A'$ be the set of vertices in $A$ with at most $2$ neighbors, we see that $A'$ is an independent set of size at least $(n-1)/6$, and so we appeal to \pref{theorem:general_lb} for the result.
\end{proof}

\begin{proof}[Proof of \pref{ex:dregular_lb}]
Fix $d\geq{}3$. We will construct a graph $G$ of size $(d+1)n$. By building up from components as follows:
\begin{itemize}
\item For each $k\in\brk{n}$ let $G_{k}$ be the complete graph on $d+1$ vertices. Remove an edge from an arbitrary pair of vertices $(u_k, v_k)$.
\item Form $G$ by taking the collection of all $G_k$, then adding an edge connecting $v_k$ to $u_{k+1}$ for each $k$, with the convention $u_{n+1}=u_1$.
\end{itemize}
This construction for $d=3$ is illustrated in \pref{fig:3regular}.

Observe that $G$ is $d$-regular. We obtain the desired result by applying \pref{theorem:system_lb} with the collection $\crl*{G_k}$ as the set system and observing that the each component $G_k$ has only two edges leaving.

\end{proof}
\begin{proof}[Proof of \pref{ex:grid}]
We first examine the case where $c=3$. Here we take the tree decomposition illustrated in \pref{fig:const_height_decomp}, where we cover the graph with overlapping $3\times{}2$ components, and take $W^{\star}=\bigcup_{v\in{}W}N_{v}$. This yields $\mincuts{W}=3$ for all components except those at the graph's endpoints. We now connect the components as a path graph and appeal to \pref{theorem:main_simple}, which implies a rate of $\wt{O}(p^{2}n)$.

When $c=\omega(1)$ we can build a decomposition as follows (informally): Produce $E'$ as in \pref{fig:grid_cut} by performing the zig-zag cut with every third row of edges, leaving only $3$ edges on the left or right side (alternating). We can now produce $T$ (a path graph) by tiling $G'$ with overlapping $3\times{}3$ components. Again, take $W^{\star}=\bigcup_{v\in{}W}N_{v}$.

We can verify that if we perform extended inference we have $\mincuts{W}=3$ for the $O(n)$ components in the interior of the graph and $\mincuts{W}=2$ for the $O(\sqrt{n})$ components at the boundary.  

The tree decomposition is illustrated in \pref{fig:grid_snake}. We have $\widths(T)=O(1)$ and $\deg_{E}(T)=O(1)$. Applying \pref{theorem:main_simple} thus gives an upper bound of $\wt{O}(p^{2}n + p\sqrt{n})$ with probability at least $1-\delta$.

Since $T$ is a path graph, we pay $O(n\ceil{p^2n})$ in computation as per \pref{app:algorithms}.

\begin{figure}[h!]
\centering
\includegraphics[scale=.5]{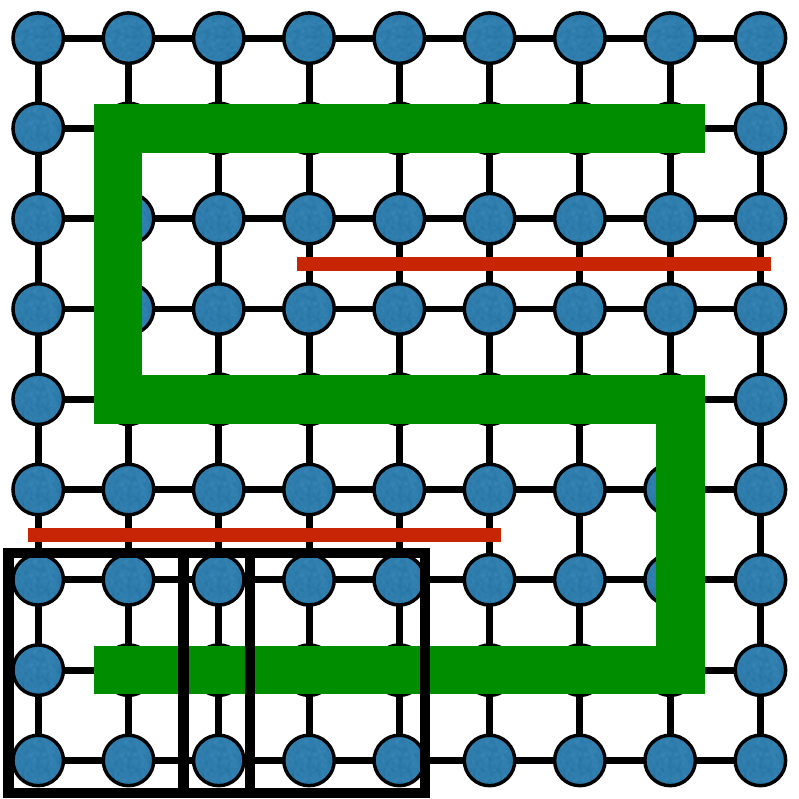}
\caption{Tree decomposition for 2D grid.}
\label{fig:grid_snake}
\end{figure}

\end{proof}

\begin{figure}[h!]
\centering
\includegraphics[scale=.2]{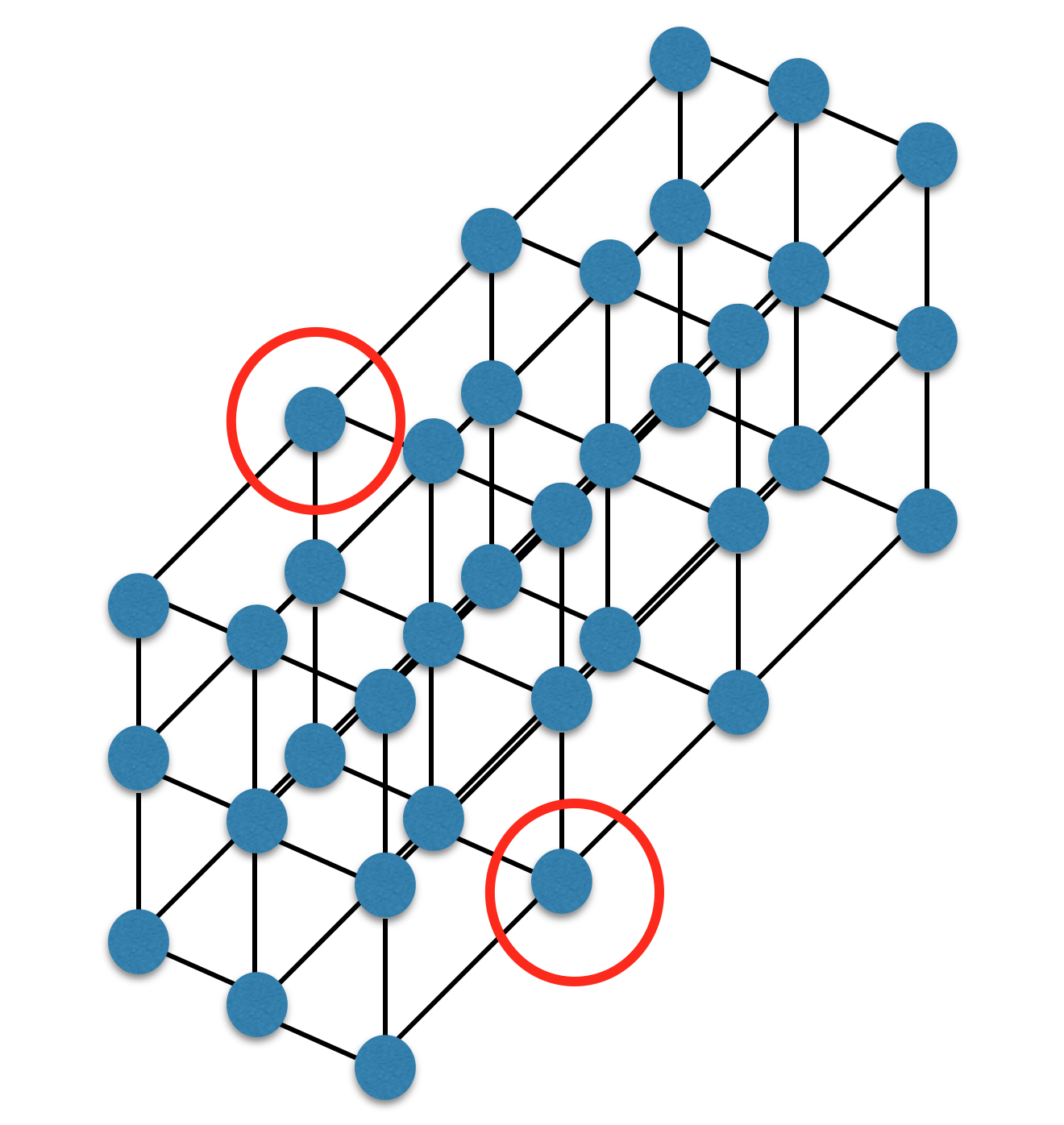}
\caption{Lower bound argument for $n/c^{2}\times{}c\times{}c$ hypergrid.}
\label{fig:hypertube}
\end{figure}

\begin{proof}[Proof of \pref{ex:hypertube}]
We will prove this result for the three-dimensional case. We first show the lower bound.

Suppose $c\geq{}3$ is constant, so that we are in the ``hypertube'' regime. Note that vertices on the outermost ``edges'' of the hypertube, examples of which are circled in \pref{fig:hypertube}, have degree at most $4$. There are $\Omega(n)$ such vertices, so appealing to \pref{theorem:general_lb} yields a lower bound on Hamming error of $\Omega(p^{2}n)$. In fact for the $n/c^2 \times c \times c$ hyper-tube one can achieve the $O(p^2 n)$ rate using our method. Simply take each components of size $2 \times c \times c$ connected in a path as in the example for the 2D grid. Since the minimum cut for each component is already at least $3$, we don't need to consider extended components and simply use brute-force on the components themselves.

We now sketch the upper bound for the $n^{1/3}\times{}n^{1/3}\times{}n^{1/3}$ hypergrid. We use a technique similar to that used for the 2D grid in \pref{ex:grid}: We take $T$ to be a path graph obtained by covering the hypergrid in overlapping $3\times3\times{}3$ components in a zig-zagging pattern. Note that each $3\times{}3\times{}3$ component will contain nodes similar to those highlighted in \pref{fig:hypertube} with degree at most 4. This means $\mincuts{W}=4$, so to obtain the $O(p^{3}n)$ Hamming error we must consider extended components. Take $W^{\star}=\bigcup_{v\in{}W}N_v$. Then $\mincuts{W}=6$ for all components except those at the boundary of the hypergrid, which have $\mincuts{W}\in\crl*{4,5}$. There are only $o(n)$ such components, so we achieve the $O(p^{3}n)$ upper bound by appealing to \pref{theorem:main_simple}.

For higher-dimensional hypergrids, the strategy of taking components to be constant-sized hypergrids and $T$ to be a zig-zagging path graph readily extends. The lower bound stated follows from a simple counting argument. 

In general, we can associated vertices of a $c_{1}\times{}c_2\times{}\ldots\times{}c_d$ hypergrid with the elements of $\Z_{c_1}\times{}\Z_{c_2}\times{}\ldots\times{}\Z_{c_d}$. For a vertex $v=(v_1,\ldots,v_d)$, the degree is given by $\deg(v)=\abs{\crl*{k\in\brk{d}\mid{}v_{k}\in\crl*{0,c_k}}}$. 

Consider the case where $c_1,\ldots{}c_{d-1}=c$, $c_{d}=n/c^{d-1}$. In this case, the degree argument above implies

\[
\abs{\crl*{v\mid{}\deg(v)=d+1}}\geq{}\sum_{v_{k}\in\crl*{0,c}:k\neq{}d}(n-2) = \Omega(n).
\]
Thus, a constant fraction of vertices have degree $d+1$, and so \pref{theorem:general_lb} implies a lower bound of $\Omega(p^{\ceil{\frac{d+1}{2}}}n)$.

\end{proof}

\begin{proof}[Proof of \pref{ex:ring_lattice}]~\\
\textbf{Upper bound: Tree decomposition}
We first formally define the tree decomposition $T=(\tW, F)$ that we will use with \pref{alg:decoder}.
Assume for simplicity what $n=n'\cdot(2k+1)$. We will define a vertex set $\crl*{v_{1},\ldots,v_{n'}}$ as follows: $v_{1}=1$, $v_{i+1}=v_i + k + 1$. We will now define a component for each of these vertices:
\[
W(v_i) = N_{G}(v_i).
\]
Let $\tW$ will be the union of these components. Since we assumed $n$ to be divisible by $(2k+1)$, the components a partition of $V$.
We now define the $\extend$ function for this decomposition:
\[
\extend(W) = \bigcup_{v\in{}W}N_{G}(v).
\]
That is, the extended component $W^{\star}(v_i)$ is the set of all vertices removed from $v_i$ by paths of length $2$.

Finally, we construct the edge set $F$ by adding edges of the form $(W(v_i), W(v_{i+1}))$ for $i\in\crl*{1,\ldots,n'-1}$. This means that the decomposition is a path graph. The decomposition is clearly admissible in the sense of \pref{def:admissible}.

We can observe that $\mincuts{W}=2k$ just as the minimum cut of $R_{n,k}$ is itself $2k$. \pref{theorem:main_simple} thus implies a recovery rate of $\Ot(p^{k}n)$. Since $T$ is a path graph, the algorithm runs in time $O(\ceil{p^{k}n}n)$.

\paragraph{Lower bound} 
That $O(p^{k}n)$ is optimal can be seen by appealing to \pref{theorem:general_lb} with the fact that $R_{n,k}$ is $2k$-regular.

\end{proof}

\begin{proof}[Proof of \pref{ex:newman_watts}]
The average number of vertices added is $\alpha{}n$. By the Chernoff bound, with high probability the number of vertices added is bounded as $\alpha{}n+c\sqrt{\alpha{}n\log{}n}$ for some constant $c$. This means that for any $\eps>0$, there is some minimum $n$ for which an $(1-\alpha+\eps)$ fraction of vertices have no edges added. This means that there are at least $(1-\alpha+\eps)n$ edges with degree $2k$, so \pref{theorem:general_lb} yields the result.
\end{proof}

\subsection{Analysis of \textsc{TreeDecompositionDecoder}}
\paragraph{Properties of Tree Decompositions}
We begin by recalling a few properties of tree decompositions that are critical for proving the performance bounds for \pref{alg:decoder}.
\begin{proposition}
\label{prop:tree_decomp_properties}
For any tree decomposition $T=(\W,F)$, the following properties hold:
\begin{enumerate}
\item For each $v\in{}V$ there exists $W$ with $v\in{}W$.\\
\indent{}\emph{This guarantees that we produce a prediction for each vertex.}
\item If $(W_1, W_2)\in{}F$, there is some $v\in{}V$ with $v\in{}W_1,W_2$.\\
\indent{}\emph{This guarantees that the class $\F$ (see \pref{eq:fclass}) is well-defined.}
\item $T$ is connected\\
\indent{}\emph{This implies that $\abs{\F}\lesssim{}2^{K}$.}
\item $\abs{\tW}\leq{}n$.\\
\emph{This implies that a mistake bound for components of the tree decomposition translates to a mistake bound for vertices of $G$}.
\end{enumerate}
\end{proposition}

\begin{proof}[Proof of \pref{prop:tree_decomp_properties}]
\begin{enumerate}
\item \pref{def:tree_decomposition}.
\item Suppose there is some edge $(W_1, W_2)\in{}F$ with no common vertices. Consider the subtrees $T_1$ and $T_2$ created by removing $(W_1, W_2)\in{}F$. By the coherence property (\pref{def:tree_decomposition}), the subgraphs of $G'$ associated with these decompositions (call them $G'_{T_1}$  and $G'_{T_2}$) must have no common nodes. Yet, $G'$ is connected, so there must be $(u,v)\in{}E'$ with $u\in{}G'_{T_1}$, $v\in{}G'_{T_2}$. Our hypothesis now implies that there is no $W\in\W$ containing $u$ and $v$, so $T$ violates the edge inclusion property of the tree decomposition.
\item \pref{def:tree_decomposition}
\item This follows directly from the non-redundancy assumption of \pref{def:tree_decomposition}. See, e.g., \cite[10.16]{kleinberg2006algorithm}.
\end{enumerate}
\end{proof}

\paragraph{Estimation in Tree Decomposition Components}

We now formally define and analyze the component-wise estimators computed in \pref{line:estimator} of \pref{alg:decoder}.
\begin{definition}[Extended Component Estimator]
\label{def:extended_component_estimator}
Consider the (edge) maximum likelihood estimator over $\Ws$:
\begin{equation}
\label{eq:extended_component_estimator}
\Yt^{\Ws}\defeq{}\argmin_{\Yt\in\pmo^{\Ws}}\sum_{uv\in{}E'(\Ws)}\ind\crl{\Yt_u\Yt_v\neq{}X_{uv}}.
\end{equation}
We define the \textbf{extended component estimator} $\Yhws\in\pmo^{W}$ as restriction of $\Yt^{\Ws}$ to $W$.
\end{definition}
For $\Yhws$ estimation performance is governed by $\mincuts{W}$ rather than $\mincut{W}$, as the next lemma shows:

\begin{lemma}[Error Probability for Extended Component Estimator]
\label{lem:extended_component_estimator}
\[
\Pr\left(\min_{s\pmo} \ind\crl{s\Yhws \ne Y^W} >0\right) \le 2^{|\Ws|} p^{\lceil\mincuts{W}/2\rceil}.
\]
\end{lemma}
\begin{proof}[Proof of \pref{lem:extended_component_estimator}]
 Suppose $\Yhws\neq{}Y^W$ and consider $D=\crl{v\in{}\Ws:\Yt^{\Ws}_v\neq{}Y_v}$. Then there is some maximal connected component $S$ of $D$ containing at least one vertex of $W$. It must then be the case that at least half the edge samples in $\delta(S)$ are flipped with respect to the ground truth. Consequently it holds that
\begin{align*}
\Pr\prn*{\min_{s\pmo} \ind\crl{s\Yhws \ne Y^W} >0}&\leq{}\sum_{S\subseteq{}\Ws:S\cap{}W\neq{}\emptyset, \bar{S}\cap{}W\neq{}\emptyset}p^{\ceil{|\delta(S)|/2}}\\ &\leq{}\sum_{S\subseteq{}\Ws}p^{\ceil{\mincuts{W}/2}}\\ &\leq 2^{\abs{\Ws}}p^{\ceil{\mincuts{W}/2}}.
\end{align*}
 \end{proof}
\pref{lem:extended_component_estimator} shows that considering $\textrm{mincut}^{\star}$ offers improved failure probability over $\textrm{mincut}$ because it allows us to take advantage of all of the information in $\Ws$, yet only pay (in terms of errors) for cuts that involve nodes in the core component $W$. In Figure \ref{fig:const_height_decomp}, all components of the tree decomposition except the endpoints have $\mincuts{W}=3$, and so their extended component estimators achieve $O(p^2)$ failure probability.

\paragraph{Concentration}

We begin by stating a concentration result for functions of independent random variables, which we will use to establish a bound on the total number of components that fail in the first stage of our algorithm. Let $X_{1},\ldots,X_n$ be independent random variables each taking values in a probability space $\mc{X}$, and let $F:\mc{X}^{n}\to{}\R$. We will be interested in the concentration of the random variable $S=F(X_1,\ldots,X_n)$. Letting $X'_1,\ldots,X'_n$ be independent copies of $X_1,\ldots,X_n$, we define $S^{(i)}=F(X_1,\ldots,X_{i-1},X'_{i},X_{i+1},\ldots,X_n)$. Finally, we define a new random variable
\[
V_{+}=\sum_{i=1}^{n}\En\brk*{(S - S^{(i)})_{+}^{2}\mid{}X_1,\ldots,X_n}.
\]
\begin{theorem}[Entropy Method with Efron-Stein Variance \citep{Boucheron2003Concentration}]
\label{theorem:entropy}
If there exists a constant $a>0$ such that $V_{+}\leq{} aS$ then
\[
\Pr\crl*{S\geq{}\En\brk{S}+t}\leq{} \exp\prn*{\frac{-t^{2}}{4a\En\brk{S} + 2at}}.
\]
Subsequently, with probability at least $1-\delta$,
\[
S \leq{} \En\brk{S} + \max\crl*{4a\log(1/\delta), 2\sqrt{2a\En\brk*{S}\log(1/\delta)}}\leq{} 2\En\brk{S} + 6a\log(1/\delta).
\]

\end{theorem}

With \pref{theorem:entropy} in mind, we may proceed to a bound on the number of components with mistakes when the basic component estimator \pref{eq:component_estimator} is used.
 \begin{lemma}[Formal Version of \pref{lem:conc_informal}]
 \label{lem:concentration_basic}
 For all $\delta>0$, with probability at least $1-\delta$ over the draw of $X$,
 {\small
 \begin{align}
 \label{eq:conc_basic1}
 \min_{s\in\pmo^{\tW}}\sum_{W\in\tW}\ind\crl{s_W\Yh^W \ne Y^W} &\leq
2\sum_{W\in\W}2^{\abs{W}}p^{\ceil{\mincut{W}/2}} +  6\max_{e\in{}E}\abs{\W(e)}\max_{W\in\W}\abs{E'(W)}\log(1/\delta).\\
 \end{align}}
     \end{lemma}
 \begin{proof}[Proof of \pref{lem:concentration_basic}]
 Define a random variable
\[
S(X) = \sum_{W\in\W}\min_{s\in\pmo}\ind\crl*{s\Yh^{W}(X)\neq{}Y^{W}},
\]
where $\Yh^{W}$ are the component-wise estimators produced by Algorithm \ref{alg:decoder} and $X$ are the edge observations. To prove the lemma we will apply \pref{theorem:entropy} by showing that there is a constant $a$ such that the necessary variance bound $V_{+}\leq{}aS$ holds.

To this end, consider
\[
S(X) - S(\Xe) = \sum_{W\in\W}\prn*{\min_{s\in\pmo}\ind\crl*{s\Yh^{W}(X)\neq{}Y^{W}}-\min_{s\in\pmo}\ind\crl*{s\Yh^{W}(\Xe)\neq{}Y^{W}}},
\]
where $\Xe$ is defined as in \pref{theorem:entropy}. To be more precise, we draw $(X'_{e})_{e\in{}E}$ from the same distribution as $X$, then let $X^{(e)}$ be the result of replacing $X_{e}$ with $X'_{e}$.

We have
\[
S(X) - S(\Xe) = \sum_{W\in\W(e)}\prn*{\min_{s\in\pmo}\ind\crl*{s\Yh^{W}(X)\neq{}Y^{W}}-\min_{s\in\pmo}\ind\crl*{s\Yh^{W}(\Xe)\neq{}Y^{W}}},
\]
since changing $X_e$ can only change $\Yh^{W}$ if $e\in{}W$. Now, since $S(\Xe)$ is nonnegative we have
\begin{align*}
(S(X) - S(\Xe)_{+}^{2}&=\prn*{\sum_{W\in\W(e)}\prn*{\min_{s\in\pmo}\ind\crl*{s\Yh^{W}(X)\neq{}Y^{W}}-\min_{s\in\pmo}\ind\crl*{s\Yh^{W}(\Xe)\neq{}Y^{W}}}}_{+}^{2}\\
&\leq{}\prn*{\sum_{W\in\W(e)}\min_{s\in\pmo}\ind\crl*{s\Yh^{W}(X)\neq{}Y^{W}}}^{2}\\
&\leq{}\abs{\W(e)}\sum_{W\in\W(e)}\min_{s\in\pmo}\ind\crl*{s\Yh^{W}(X)\neq{}Y^{W}}.
\end{align*}
We now sum over all edges to arrive at an upper bound on $V_+$:
\begin{align*}
V_{+}&=\sum_{e\in{}E}\En\brk*{(S(X) - S(\Xe)_{+}^{2}\mid{}X}\\
&\leq{} \max_{e\in{}E}\abs{\W(e)}\sum_{e\in{}E}\sum_{W\in\W(e)}\min_{s\in\pmo}\ind\crl*{s\Yh^{W}(X)\neq{}Y^{W}}\\
&= \max_{e\in{}E}\abs{\W(e)}\sum_{W\in\W}\sum_{e\in{}E(W)}\min_{s\in\pmo}\ind\crl*{s\Yh^{W}(X)\neq{}Y^{W}}\\
&\leq{} \max_{e\in{}E}\abs{\W(e)}\max_{W\in\W}\abs{E(W)}\sum_{W\in\W}\min_{s\in\pmo}\ind\crl*{s\Yh^{W}(X)\neq{}Y^{W}}\\
&\leq{} \max_{e\in{}E}\abs{\W(e)}\max_{W\in\W}\abs{E(W)}\sum_{W\in\W}\min_{s\in\pmo}\ind\crl*{s\Yh^{W}(X)\neq{}Y^{W}}\\
&= \max_{e\in{}E}\abs{\W(e)}\max_{W\in\W}\abs{E(W)}S(X).
\end{align*}
We now appeal to \pref{theorem:entropy} with $a=\max_{e\in{}E}\abs{\W(e)}\max_{W\in\W}\abs{E(W)}$, which yields that with probability at least $1-\delta$,
\[
S\leq{} 2\En\brk*{S} + 6\max_{e\in{}E}\abs{\W(e)}\max_{W\in\W}\abs{E(W)}\log(1/\delta).
\]
Finally, the bound on $\En\brk{S}$ follows from \pref{prop:component_estimator}:
\[
\En\brk*{S}= \sum_{W\in\W}\Pr\prn*{\min_{s\in\pmo}\ind\crl*{s\Yh^{W}(X)\neq{}Y^{W}}} \leq{} \sum_{W\in\W}2^{|W|} p^{\lceil\mincut{W}/2\rceil}.
\]

\end{proof}

An analogous concentration result to \pref{lem:concentration_basic} holds to bounds the number of components that fail over the whole graph when the extended component estimator is used:
 \begin{lemma}
 \label{lem:concentration_extended}
 For all $\delta>0$, with probability at least $1-\delta$ over the draw of $X$,
  {\small
 \begin{align}
 \label{eq:conc_extended1}
 \min_{s\in\pmo^{\tW}}\sum_{W\in\tW}\ind\crl{s_W\Yh^{W^{\star}} \ne Y^{W^{\star}}} &\leq
2\sum_{W\in\W}2^{\abs{W^{\star}}}p^{\ceil{\mincuts{W}/2}} +  6\max_{e\in{}E}\abs{\W^{\star}(e)}\max_{W\in\W}\abs{E'(W^{\star})}\log(1/\delta).
 \end{align}
 }

 \end{lemma}

 \begin{proof}[Proof of \pref{lem:concentration_extended}]
 This proof proceeds exactly as in the proof of \pref{lem:concentration_basic} using
 \[
 S(X) =\sum_{W\in\W}\min_{s\in\pmo}\ind\crl*{s\Yh^{W^{\star}}(X)\neq{}Y^{W}}.
 \]
 The only difference is that edges are more influential than in that lemma because each extended component estimator $\Yh^{W^{\star}}$ may depend on more edges than the simpler component estimator $\Yh^{W}$. To this end, define $\mc{W}^{\star}(e)=\crl*{W\mid{}e\in{}E'(W^{\star})}$. One can verify that if we replace every instance of $\mc{W}(e)$ in the proof of \pref{lem:concentration_basic} with $\mc{W}^{\star}(e)$ it holds that $V_{+}\leq{}aS$ with $a=\max_{e\in{}E}\abs{\W^{\star}(e)}\max_{W\in\W}\abs{E(W^{\star})}$. \pref{theorem:entropy} then implies that with probability at least $1-\delta$,
 \begin{align*}
 S&\leq{} 2\En\brk*{S} + 6\max_{e\in{}E}\abs{\W^{\star}(e)}\max_{W\in\W}\abs{E(W^{\star})}\log(1/\delta)\\
 &= 2\En\brk*{S} + 6\deg^{\star}_E(T)\max_{W\in\W}\abs{E(W^{\star})}\log(1/\delta).
 \end{align*}
 \end{proof}

\begin{proof}[Proof of \pref{theorem:main_simple}]~\\
\textbf{Full theorem statement} We will prove the following error bound:  If $T=(\mc{W}, F)$ is admissible, with probability at least $1-\delta$ over the draw of $X$ and $Z$, $\Yh$ satisfies:
{\small
\begin{align}
\label{eq:main_bound}
&\sum_{v\in{}V}\ind\crl*{\Yh_v\neq{}Y_v} \\
&\leq{} O\prn*{
\frac{1}{\eps^{2}}
\prn*{
2^{\widths(T)}\sum_{W\in\tW}p^{\ceil{\mincuts{W}/2}}
+\deg^{\star}_E(T)\max_{W\in\W}\abs{E(W^{\star})}\log(1/\delta) 
}\cdot{}\prn*{\width(T) + \deg(T)\log{}n}
}
\end{align}
}
This statement specializes to \pref{eq:main_w} when all of the tree decomposition quantities are constant and $\delta=1/n$.

\paragraph{Error bound for individual components} \pref{lem:extended_component_estimator} implies that for a fixed component $W\in\tW$, the probability that the estimator produced by the brute-force enumeration routine fails to exactly recover the labels in $W$ (up to sign) is bounded as
\[
\Pr\prn*{\min_{s\pmo} \ind\crl{s\Yhws \ne Y^W} >0} \leq{} 2^{\abs{\Ws}}p^{\ceil{\mincuts{W}/2}}.
\]
\paragraph{Error bound across all components}
Consider the following random variable, which is the total number components
 \[
 S(X) =\sum_{W\in\W}\min_{s\in\pmo}\ind\crl*{s\Yh^{W^{\star}}(X)\neq{}Y^{W}}.
 \]
 The bound on component failure probability immediately implies in in-expectation bound on $S$:
 \[
 \En\brk*{S}\leq{} \sum_{W\in\tW}2^{\abs{\Ws}}p^{\ceil{\mincuts{W}/2}}.
 \]
 \pref{lem:concentration_extended} shows that $S$ concentrates tightly around its expectation. More precisely, let $A \defeq 6\deg^{\star}_E(T)\max_{W\in\W}\abs{E(W^{\star})}$ and
\begin{equation}
\label{eq:kn}
K_n \defeq 2^{\widths(T)+2}\sum_{W\in\tW}p^{\ceil{\mincuts{W}/2}} + A\log(2/\delta).
\end{equation}
Then \pref{lem:concentration_extended} implies that with probability at least $1-\delta/2$,

\begin{align}
 \min_{s\in\pmo^{\tW}}\sum_{W\in\tW}\ind\crl{s_W\Yhws \ne Y^W}
 &\leq{} 2\sum_{W\in\tW}2^{\abs{\Ws}}p^{\ceil{\mincuts{W}/2}} + A\log(2/\delta)\notag
 \\
 &\leq{} K_n\label{eq:kn_bound}
\end{align}
\paragraph{Inference with side information: Hypothesis class}
Consider the following binary signing of the components in $T$:
\[
s^{\star} = \argmin_{s\in\pmo^{\tW}}\sum_{W\in\tW}\ind\crl{s_W\Yhws \ne Y^W}.
\]
$s^{\star}$ is signing of the component-wise predictions $(\Yh^{\Ws})$ that best matches the ground truth. If we knew the value of $s^{\star}$ we could use it to produce a vertex prediction with at most $K_n$ mistakes. Computing the $s^{\star}$ is information-theoretically impossible because we do not have access to $Y$, but we will show that the signing we produce using the side information $Z$ is close.

Define
\begin{equation}
\label{eq:ln}
L_n=\deg(T)\cdot{}K_n.
\end{equation}
We will argue that \pref{eq:kn_bound} implies that $s^{\star}$ lies in the class
\begin{equation}
\label{eq:fclass}
\F(X) \defeq{}  \crl*{s\in\pmo^{\tW}\mid\sum_{(W_1,W_2) \in F} \ind\crl{s_{W_1}  \ne s_{W_2} \cdot S(W_1,W_2) } \le L_n}.
\end{equation}
First, consider the \texttt{for} loop on \pref{alg:decoder}, \pref{line:decoder_meta_edge}. \pref{prop:tree_decomp_properties} implies that $S(W_1, W_2)$ as defined in this loop is well-defined, because there always exists some $v\in{}W_1\cap{}W_2$.

Second, consider the value of
\[
\sum_{(W_1,W_2) \in F} \ind\crl{s^{\star}_{W_1}  \ne s^{\star}_{W_2} \cdot S(W_1,W_2) } = \sum_{(W_1,W_2) \in F} \ind\crl{s^{\star}_{W_1}  \ne s^{\star}_{W_2} \cdot \Yh^{\Ws_1}_v\cdot{}\Yh^{\Ws_2}_v }.
\]
We can bound this quantity in terms of the number of components $W$ for which
\[\min_{s\in\pmo}\ind\crl*{s\Yh^{W^{\star}}\neq{}Y^{W}}=1.
\]
Observe that if $\min_{s\in\pmo}\ind\crl*{s\Yh^{W^{\star}}\neq{}Y^{W}}=0$ then there is some $\bar{s}_{W}\in\pmo$ such that $\Yh^{W^{\star}}=\bar{s}_{W}Y^{W}$. If we take $s^{\star}_{W}=\bar{s}_{W}$ in all the components with no errors, and choose the sign arbitrarily for others, we will have $\ind\crl{s^{\star}_{W_1}  \ne s^{\star}_{W_2} \cdot \Yh^{\Ws_1}_v\cdot{}\Yh^{\Ws_2}_v }=0$ whenever both $W_1$ and $W_2$ have no errors. Pessimistically, there are at most $L_n=\deg(T)\cdot{}K_n$ edges $(W_1, W_2)$ where at least one of $W_1$ or $W_2$ has an error, and therefore \pref{eq:kn_bound} implies that with probability at least $1-\delta/2$, $s^{\star}\in\F$.

We conclude this discussion by showing that $\abs{\mc{F}(X)}$ small. Since by \pref{prop:tree_decomp_properties} $T$ is connected, labelings of the edges of $T$ are in one to one correspondence with labelings of the components. Consequently,
\begin{equation}
\label{eq:f_size}
\abs{\F(X)}\leq{} \sum_{k=0}^{L_n}\binom{\abs{\tW}}{k} \leq{} \prn*{\frac{e\abs{\tW}}{L_n}}^{L_n}\leq{} \prn*{\frac{en}{L_n}}^{L_n}.
\end{equation}
The last inequality uses that, from \pref{prop:tree_decomp_properties}, $\abs{\tW}\leq{}n$.

\paragraph{Final error bound for inference with side information}
We now use the properties of $\F(X)$ to derive an error bound for the prediction $\Yh$.
Recall from \pref{alg:decoder} that $\Yh$ is defined in terms of
\begin{equation}
\label{eq:whatev}
 \hat{s}=\min_{s\in\F(X)}\sum_{W \in \mathcal{W}} \sum_{v \in W} \ind\crl{s_W \Yh^{\Ws}_v \ne Z_v}.
 \end{equation}
 We reduce the analysis of the error rate of $\hat{s}$ to analysis of excess risk in a manner that parallels the proof of \pref{theorem:tree_decoding}, but is slightly more involved because the best predictor in $\F$ does not perfectly match the ground truth.
Fix $\hat{s}\in\pmo^{\tW}$. For each component $W\in\tW$ we have
\begin{align*}
\sum_{v \in W} \ind\crl{\hat{s}_W \Yh^{\Ws}_{v} \ne Y_v} & \le \sum_{v \in W}  \ind\crl{\hat{s}_W \Yh^{\Ws}_{v} \ne s^*_W \Yh^{\Ws}_{v}} + \sum_{v \in W}  \ind\crl{s^*_W \Yh^{\Ws}_{v} \ne Y_v}\\
& \le \sum_{v \in W} \ind\crl{\hat{s}_W \Yh^{\Ws}_{v} \ne s^*_W \Yh^{\Ws}_{v}} + |W|\ \ind\crl{s^*_W \Yh^{\Ws} \ne Y^W} \\
&= \frac{1}{1 - 2q} \sum_{v \in W : s^*_W \Yh^{\Ws}_{v} = Y_v} \left(\Pr_{Z}\left( \hat{s}_W \Yh^{\Ws}_{v} \cdot Z_v < 0\right) -  \Pr_Z\left(s^*_W\Yh^{\Ws}_{v} \cdot Z_v < 0\right) \right)   \\
& ~~~~~- \frac{1}{1 - 2q}    \sum_{v \in W : s^*_W \Yh^{\Ws}_{v} \ne Y_v}\left( \Pr_Z\left( \hat{s}_W \Yh^{\Ws}_{v} \cdot Z_v < 0\right) -  \Pr_Z\left(s^*_W \Yh^{\Ws}_{v} \cdot Z_v < 0\right) \right) \\
&~~~~~ + |W|\  \ind\crl{s^*_W \Yh_W \ne Y_W}.
\end{align*}
Now note that given that $Z_v$ is drawn as a noisy version of $Y_v$, \\ $\left|\Pr_{Z}\left( \hat{s}_W \Yh^{\Ws}_{v} \cdot Z_v < 0\right) -  \Pr_{Z}\left(s^*_W \Yh^{\Ws}_{v} \cdot Z_v < 0\right)\right| = 1 - 2q$ and so
{\small
\begin{align*}
 - \frac{1}{1 - 2q} &   \sum_{v \in W : s^*_W \Yh^{\Ws}_{v} \ne Y_v}\left( \Pr_{Z}\left( \hat{s}_W \Yh^{\Ws}_{v} \cdot Z_v < 0\right) -  \Pr_{Z}\left(s^*_W \Yh^{\Ws}_{v} \cdot Z_v < 0\right) \right)  \\
 &\le 2 \sum_{v \in W}  \ind\crl{s^*_W \Yh^{\Ws}_{v} \ne Y_v} + \frac{1}{1 - 2q}    \sum_{v \in W : s^*_W \Yh^{\Ws}_{v} \ne Y_v}\left( \Pr_{Z}\left( \hat{s}_W \Yh^{\Ws}_{v} \cdot Z_v < 0\right) -  \Pr_{Z}\left(s^*_W \Yh^{\Ws}_{v} \cdot Z_v < 0\right) \right) \\
 & \le 2 |W| \ind\crl{s^*_W \Yhws \ne Y^W} + \frac{1}{1 - 2q}    \sum_{v \in W : s^*_W \Yh^{\Ws}_{v} \ne Y_v}\left( \Pr_{Z}\left( \hat{s}_W \Yh^{\Ws}_{v} \cdot Z_v < 0\right) -  \Pr_{Z}\left(s^*_W \Yh^{\Ws}_{v} \cdot Z_v < 0\right) \right).
\end{align*}}
We conclude that
\begin{align*}
&\sum_{v \in W} \ind\crl{\hat{s}_W \Yh^{\Ws}_{v} \ne Y_v}\\
& \leq{} 3 |W| \ind\crl{s^*_W \Yhws \ne Y^W} + \frac{1}{1 - 2q} \sum_{v \in W} \left(\Pr_{Z}\left( \hat{s}_W \Yh^{\Ws}_{v} \cdot Z_v < 0\right) -  \Pr_Z\left(s^*_W\Yh^{\Ws}_{v} \cdot Z_v < 0\right) \right).
\end{align*}
Summing over all the components $W\in\tW$ we arrive at the bound
{\small
\begin{align*}
\sum_{W \in \mathcal{W}} & \sum_{v \in W} \ind\crl{\hat{s}_W \Yh^{\Ws}_{v} \ne Y_v} \\
& \le 3 \left(\max_{W \in \mathcal{W}} |W|\right) \sum_{w \in \mathcal{W}}\ind\crl{s^*_W \Yhws \ne Y^W} + \frac{1}{1 - 2q}    \sum_{W\in\tW}\sum_{v \in W}\left( \Pr_{Z}\left( \hat{s}_W \Yh^{\Ws}_{v} \cdot Z_v < 0\right) -  \Pr_{Z}\left(s^*_W \Yh^{\Ws}_{v} \cdot Z_v < 0\right) \right)\\
& \le 3 \left(\max_{W \in \mathcal{W}} |W|\right)K_n+ \frac{1}{1 - 2q}  \sum_{W\in\tW}\sum_{v \in W}\left( \Pr_{Z}\left( \hat{s}_W \Yh^{\Ws}_{v} \cdot Z_v < 0\right) -  \Pr_{Z}\left(s^*_W \Yh^{\Ws}_{v} \cdot Z_v < 0\right) \right)
\end{align*}}

We can now appeal to the statistical learning bounds from \pref{app:statistical_learning} to handle the right-hand side of this expression. \pref{lem:fast_rate} implies that if we take $\hat{s} = \argmin_{s\in\F}\sum_{W\in\tW}\sum_{v \in W}\ind\crl*{\hat{s}_W \Yh^{\Ws}_{v} \cdot Z_v < 0}$, which is precisely the solution to \pref{eq:whatev}, we obtain the excess risk bound,
\begin{align*}
&\sum_{W\in\tW}\sum_{v \in W}\left( \Pr_{Z}\left( \hat{s}_W \Yh^{\Ws}_{v} \cdot Z_v < 0\right) -  \Pr_{Z}\left(s^*_W \Yh^{\Ws}_{v} \cdot Z_v < 0\right) \right) \\
&\leq{} \prn*{\frac{2}{3} + \frac{c}{2}}\log(2\abs{\F}/\delta) + \frac{1}{c}\sum_{w\in\tW}\sum_{v\in{}W}\ind\crl{\hat{s}_W \Yh^{\Ws}_{v} \ne Y_v},
\end{align*}
with probability at least $1-\delta/2$ over $Z$ for all $c>0$. If we choose $c=1/\eps$, rearrange, and apply the union bound, this implies that with probability at least $1-\delta$ over the draw of $X$ and $Z$ we have
\[
\sum_{W \in \mathcal{W}}  \sum_{v \in W} \ind\crl{\hat{s}_W \Yh^{\Ws}_{v} \ne Y_v}\leq{} 6 \left(\max_{W \in \mathcal{W}} |W|\right)K_n + \frac{2}{\eps^{2}}\log(2\abs{\F}/\delta).
\]
Recall that $\abs{\F}\leq{} (e\abs{\tW}/L_n)^{L_n}$, which implies a bound of

\begin{align*}
&\sum_{W \in \mathcal{W}}  \sum_{v \in W} \ind\crl{\hat{s}_W \Yh^{\Ws}_{v}\neq{}Y_v} \\
&\leq{}
O\prn*{
\frac{1}{\eps^{2}}\brk*{
\width(T)\cdot{}K_n + L_n\cdot{}\log(en/L_n) + \log(1/\delta)
}
}\\
&\leq{}
O\prn*{
\frac{1}{\eps^{2}}\brk*{
K_n\cdot{}\prn*{\width(T) + \deg(T)\cdot{}\log(en/K_n)} + \log(1/\delta)
}
}\\
&\leq{}
O\prn*{
\frac{1}{\eps^{2}}
\prn*{
2^{\widths(T)}\sum_{W\in\tW}p^{\ceil{\mincuts{W}/2}}
+\deg^{\star}_E(T)\max_{W\in\W}\abs{E(W^{\star})}\log(1/\delta)
}\cdot{}\prn*{\width(T) + \deg(T)\log{}n}
}
\end{align*}

Our choice of $\Yh$ in \pref{alg:decoder} ensures that the Hamming error $\sum_{v\in{}V}\ind\crl*{\Yh_v\neq{}Y_v}$ inherits this bound. \pref{prop:tree_decomp_properties} implies that every $v\in{}V$ is in some component, so this choice is indeed well-defined.
\end{proof}


\section{Statistical learning}
\label{app:statistical_learning}
Here we consider a fixed design variant of the statistical learning setting. Fix an input space $\mc{X}$ and output space $\mc{Z}$. We are given a fixed set $X_1,\ldots,X_n\in\mc{X}$ and samples $Z_1,\ldots, Z_n\in\mc{Z}$ with $Z_i$ drawn from $P(Z_i\mid{}X_i)$ for some distribution $P$. We fix a hypothesis class $\F$ which is some subset of mappings from $\mc{X}$ to $\mc{Z}$, and we would like to use $Z$ to find $\Yh\in\F$ that will predict future observations of $Z$ on $X$. To evaluate prediction we define a loss function $\ls:\mc{Z}\times{}\mc{Z}\to\R_{+}$, and define $L_{i}(Y) = \En_{Z\mid{}X_i}\brk*{\ls(Y, Z)}$. Our goal is to use $Z$ to select $\Yh\in\F$ to guarantee low \emph{excess risk}:
\begin{equation}
\sum_{i\in\brk{n}}L_i(\Yh(X_i)) - \min_{Y\in\F}\sum_{i\in\brk{n}}L_i(Y(X_i)).
\end{equation}
Typically this is accomplished using the \emph{empirical risk minimizer} (ERM): 
\[
\Yh= \argmin_{Y\in\F}\sum_{i\in\brk{n}}\ls(Y(X_i), Z_i)\footnote{There are a many standard bounds quantifying the performance of ERM in settings beyond the one we consider. See \cite{Bousquet2004Introduction} for a survey.}.\]

In this paper we consider a specific instantiation of the above framework in which
\begin{itemize}
\item $\X=V$, the vertex set for some graph (possibly a tree decomposition), and $X_1,\ldots,X_n$ are an arbitrary ordering of $V$ (so $n=\abs{V}$). In light of this we index all variables using $V$ going forward.
\item  $\mc{Z}=\pmo$. We fix $Y\in\vs$ and let $Z_v=Y_v$ with probability $1-q$ and $Z_v=-Y_v$ otherwise (as in \pref{mod:edge_vertex_measurement}).
\item $\ls(Y,Z) = \ind\crl{Y\neq{}V}$, so $L_i(Y) = \Pr_{Z}(Y\neq{}Z_v)$.
\item $\F\subseteq{}\vs$ is arbitrary.
\end{itemize}
For this setting the excess risk for a predictor $\Yh\in\vs$ can be written as 
\begin{equation}
\label{eq:erm_binary}
\sum_{v\in{}V}\Pr(\wh{Y}_v \ne Z_v) - \min_{Y'\in{}\F}\sum_{v\in{}V}\Pr\left(Y'_v \ne Z_v\right),
\end{equation}
and the empirical risk minimizer is given by $\wh{Y}=\argmin_{Y'\in\F}\sum_{v\in{}V}\ind\crl{Y'_v\neq{}Z_v}$.

We assume this setting exclusively for the remainder of the section.

\begin{lemma}[Excess risk bound for ERM]
\label{lem:fast_rate}
Let $\Yh$ be the ERM and let $\Ys=\argmin_{Y'\in{}\F}\sum_{v\in{}V}\Pr\prn*{Y'\neq{}Z}$. Then with probability at least $1-\delta$ over the draw of $Z$,
\begin{equation}
\label{eq:fast_rate_general}
\sum_{v\in{}V}\Pr\left(\wh{Y}_v \ne Z_v\right) - \min_{Y'\in{}\F}\sum_{v\in{}V}\Pr\left(Y'_v \ne Z_v\right) \leq{} \prn*{\frac{2}{3}+\frac{c}{2}}\log\prn*{\frac{\abs{\F}}{\delta}} + \frac{1}{c}\sum_{v\in{}V}\ind\crl*{\Yh_v\neq{}\Ys_v}
\end{equation}
for all $c>0$.
\end{lemma}
\begin{corollary}[ERM excess risk: Well-specified case]
\label{cor:fast_rate_well_specified}
When $Y\in\F$ we have that with probability at least $1-\delta$,
\begin{equation}
\label{eq:fast_rate_well_specified}
\sum_{v\in{}V}\Pr\left(\wh{Y}_v \ne Z_v\right) - \min_{Y\in{}\F}\sum_{v\in{}V}\Pr\left(Y_v \ne Z_v\right) \leq{}
\prn*{\frac{4}{3} + \frac{1}{\eps}}\log\prn*{\frac{\abs{\F}}{\delta}},
\end{equation}
recalling $q=1/2-\eps$.
\end{corollary}
\begin{proof}[Proof of \pref{cor:fast_rate_well_specified}]
When $Y\in{}\F$, $\Ys=Y$, and we have
\[
\sum_{v\in{}V}\ind\crl*{\Yh_v\neq{}Y_v} = \frac{1}{1-2q}\sum_{v\in{}V}\prn*{\Pr\prn*{\Yh_v\neq{}Z_v} - \Pr\prn*{Y_v\neq{}Z_v}}.
\]
Applying this inequality to the right hand side of \pref{eq:fast_rate_general} and rearranging yields
\[
\prn*{1-\frac{1}{c(1-2q)}}\sum_{v\in{}V}\prn*{\Pr\prn*{\Yh_v\neq{}Z_v} - \Pr\prn*{Y_v\neq{}Z_v}} \leq{} \prn*{\frac{2}{3}+\frac{c}{2}}\log\prn*{\abs{\F}/\delta}.
\]
To complete the proof we take $c=\frac{2}{1-2q}$, which gives
\[
\frac{1}{2}\sum_{v\in{}V}\prn*{\Pr\prn*{\Yh_v\neq{}Z_v} - \Pr\prn*{Y_v\neq{}Z_v}} \leq{} \prn*{\frac{2}{3}+\frac{1}{1-2q}}\log\prn*{\abs{\F}/\delta}.
\]
\end{proof}

\begin{proof}[Proof of \pref{lem:fast_rate}]
We will use \pref{lem:maximal_inequality} with $\mathcal{F}$ as the index set so that every $i \in [N]$ corresponds to one $Y' \in \mathcal{F}$. We define our collection of random variables as 
$$
T^{Y'}_v = \ind\{Y'_v \ne Z_v\} - \ind\{\Ys_v \ne Z_v\}
$$
where $Y$ is the ground truth and $Y'$ is any element of $\F$. Now using \pref{lem:maximal_inequality} and recalling $\sigma_{Y'}^2=\sum_{v\in{}V}\textrm{Var}(T_{v}^{Y'})$, we have that with probability at least $1 - \delta$, simultaneously for all $Y'$,
\begin{align*}
\sum_{v\in{}V} (\En[T^{Y'}_v] - T^{Y'}_v)  &
\leq{} \frac{2}{3}\log\prn*{\abs{\F}/\delta} + \sqrt{2\sigma_{Y'}^2\log\prn*{\abs{\F}/\delta}}\\
& \leq{} \inf_{c>0}\brk*{\prn*{\frac{2}{3}+\frac{c}{2}}\log\prn*{\abs{\F}/\delta} + \sigma_{Y'}^2/c}\\
& \leq{} \inf_{c>0}\brk*{\prn*{\frac{2}{3}+\frac{c}{2}}\log\prn*{\abs{\F}/\delta} + \frac{1}{c}\sum_{v\in{}V}\En\brk{(T_{v}^{Y'})^2}}.
\end{align*}

In particular this implies that for $\Yh=\argmin_{Y'\in\F}\sum_{v\in{}V}\ind\crl*{Y'_v\neq{}Z_v}$  we have that for all $c>0$, 
\begin{align*}
\sum_{v\in{}V} \left(\Pr\left(\wh{Y}_v \ne Z_v\right) - \Pr\left(\Ys_v \ne Z_v\right)\right)  \le& \sum_{v\in{}V} \left(\ind\left\{\wh{Y}_v \ne Z_v\right\} - \ind\left\{\Ys_v \ne Z_v\right\}\right) + \prn*{\frac{2}{3}+\frac{c}{2}}\log\prn*{\abs{\F}/\delta} \\
& + \frac{1}{c}\sum_{v\in{}V} \En\left[\left(\ind\{\wh{Y}_v \ne Z_v\} - \ind\{\Ys_v \ne Z_v\}\right)^2\right].
\end{align*}

Now since $\Ys \in \mathcal{F}$ and $\Yh$ is the ERM, we get that $\sum_{v\in{}V} \left(\ind\left\{\wh{Y}_v \ne Z_v\right\} - \ind\left\{\Ys_v \ne Z_v\right\}\right)  \le 0$ and so,
\begin{align*}
\sum_{v\in{}V} \left(\Pr\left(\wh{Y}_v \ne Z_v\right) - \Pr\left(\Ys_v \ne Z_v\right)\right) & \leq{} \prn*{\frac{2}{3}+\frac{c}{2}}\log\prn*{\abs{\F}/\delta} + \frac{1}{c}\sum_{v\in{}V} \En\left[\left(\ind\{\wh{Y}_v \ne Z_v\} - \ind\{\Ys_v \ne Z_v\}\right)^2\right] \\
& = \prn*{\frac{2}{3}+\frac{c}{2}}\log\prn*{\abs{\F}/\delta} + \frac{1}{c}\sum_{v\in{}V}\ind\crl*{\Yh_v\neq{}\Ys_v}.
\end{align*}

\end{proof}

\begin{lemma}[Maximal Inequality]
\label{lem:maximal_inequality}
For each $i\in\brk{N}$, let $\crl{T_v^i}_{v\in{}V}$ be a random process with each variable $T_v^i$ bounded in absolute value by $1$. Define $\sigma_i^2 = \sum_{v\in{}V} \mathrm{Var}(T_v^i)$. With probability at least $1-\delta$, 
\begin{equation}
\sum_{v\in{}V} (\En[T^{i}_v] - T^{i}_v)  \leq{} \frac{2}{3}\log\prn*{N/\delta} + \sqrt{
2\sigma_i^2\log\prn*{N/\delta}
}\quad\forall{}i\in\brk{N}.
\end{equation}
\end{lemma}
\begin{proof}[Proof of \pref{lem:maximal_inequality}]
Let us start by writing out the Bernstein bound for the random variable $\sum_{t=1}^{n}Z_{t}^i$:
\[
\Pr\left(\sum_{v\in{}V} (\En[T_v^i] - T_v^i)  > \theta\right) \le \exp\left(- \frac{\theta_i^2}{2 \sigma_i^2 + \frac{2}{3}\theta_i } \right).
\]

We now consider the family of processes $\crl{T^i_v}_{v\in{}V}$ and see that by union bound we have
\[
\Pr\left(\max_{i \in [N]} \sum_{v\in{}V} (\En[T^{i}_v] - T^{i}_v)  - \theta_{i} > 0\right) \le \sum_{i \in [N]} \exp\left(- \frac{\theta_{i}^2}{2  \sigma_{i}^2 + \frac{2}{3}\theta_i} \right).
\]
Solving the quadratic formula, it holds that if we take
\[
\theta_i \geq \frac{1}{3}\log\prn*{N/\delta} + \sqrt{
\log^2\prn*{N/\delta}/9 + 2\sigma_i^2\log\prn*{N/\delta}
},
\]
then we have
\[
\sum_{i \in [N]} \exp\left(- \frac{\theta_{i}^2}{2  \sigma_{i}^2 + \frac{4}{3} } \right) \leq{} \delta.
\]
We can conclude that
\[
\Pr\left(\forall i \in [N],~~ \sum_{v\in{}V} (\En[T^{i}_v] - T^{i}_v)  > \frac{1}{3}\log\prn*{N/\delta} + \sqrt{
\log^2\prn*{N/\delta}/9 + 2\sigma_i^2\log\prn*{N/\delta}
}\right) \le \delta.
\]
\end{proof}

\section{Algorithms}
\label{app:algorithms}
The tree inference algorithm from \pref{sec:trees} and the full tree decomposition inference algorithm, \pref{alg:decoder}, rely on the solution of a constrained minimization problem over the edges and vertices of a tree $T$. This minimization problem is stated in its most general form as \pref{alg:tree_dp2}. This problem can be solved efficiently using the following tree-structured graphical model:
\begin{itemize}
\item Fix an arbitrary order on $T$, and let $p(v)$ denote the parent of a vertex $v$ under this order.
\item Define variables $s\in\pmo^{V}$ and $C\in\crl{1,\ldots{},K_n}^{V}$.
\item For each variable $v\in{}V$ define factor:
\[
\psi_{v}(s_{v}, s_{p(v), C_{v}, C_{\delta_{+}(v)}}) = e^{-\ind\crl*{\mathrm{Cost}_{v}[s_v]}}\cdot{}\ind\crl*{\sum_{u\in\delta_{+}(v)}C_{u}\leq{}C_{v}-\ind\crl*{s_{v}\neq{}s_{p(v)}\cdot{}S(v,p(v))}}.
\]
\end{itemize}
With this formulation it is clear that given $(s,C)$ maximizing the potential
\[
\psi(s, C) = \prod_{v\in{}V}\psi_{v}(s_{v}, s_{p(v), C_{v}, C_{\delta_{+}(v)}})
\]
the node labels $s$ are a valid solution for \pref{alg:tree_dp2}. Since $\psi$ is a tree-structured MRF the maximizer can be calculated exactly using max-sum message passing (see e.g. \cite{cowell2006probabilistic}). The only catch is that naively this procedure's running time will scale as $n^{\deg(T)}$, because each of the variables $C_{v}$ has a range that scales with $n$. For example, the range of $C_v$ is $\Ot(pn)$ for the setup in \pref{sec:trees}. We now show that the structure of the factors can be exploited to perform message passing in polynomial time in $\deg(T)$ and $n$. In particular, message passing can be performed in time time $\tilde{O}({K_n}{}n^2)$ for general trees and time $\tilde{O}({K_n{}}n)$ when $T$ is a path graph.

\begin{algorithm}[h]
\textbf{Input:} Tree $T=(V,E)$, $\crl{\cost_v}_{v\in{}V}$, $\crl{S(u,v)}_{(u,v)\in{}E}$, $K_n\in\N$.
\begin{spacing}{.5}
\begin{flalign*}
 &\hat{s} = \argmin_{s \in \{\pm 1\}^{V}} \sum_{v \in V} \cost_{v}\brk{s_v}&\\
&\textrm{ s.t. }  \sum_{(u,v)\in{}E} \ind\crl{s_{u}  \ne s_{v} \cdot S(u,v) } \le K_n&
\end{flalign*}
\end{spacing}
\textbf{Return:} $\hat{s}\in\vs$.
\caption{\treealg{}}
\label{alg:tree_dp2}
\end{algorithm}
To solve \textsc{TreeDecoder} efficiently, we first turn $T$ into a DAG by running a BFS from a given vertex $r$ and directing edges according to the time of discovery. We denote this DAG by
$\overrightarrow{T}$.
We root this directed tree at $r$, and denote the parent of a vertex $u \neq r$ by $p(u)$. For $u \in V$, let $\overrightarrow{T}_u$ denote the (directed) subtree rooted at $u$. Given a labeling $Y$ to the vertices of $T$, an edge $uv$ for which $s_{u}  \ne s_{v} \cdot S(u,v)$ is called a \emph{violated} edge.

We now define a table $OPT$ that will be used to store values for sub-problems of \pref{alg:tree_dp2}. For $u \neq r$, and budget $K$, we define $OPT(u,K|1)$ to be the optimal value of the optimization problem in \pref{alg:tree_dp2} over the subtree  $\overrightarrow{T}_u$ for budget $K$, where the label of $p(u)$ is constrained to have value 1. Importantly, the edge $(u, p(u))$ is also considered in the count of violated edges (in addition to the edges in $\overrightarrow{T}_u$). $OPT(u,K|-1)$ is defined likewise, but for $p(u)$ constrained to label value $-1$.

$$OPT(u,K|1)=\min_{s \in \{-1,1\}}\min_{\sum_{v \in N_u} K_v=K-\ind\crl{s  \ne S_{p(u)} \cdot S(u,p(u)) }}\left(\sum_{v \in N(u)}OPT(v,K_v|s)+\cost_{v}[s]\right).$$

Here $s$ is simply the value assigned to $u$. We constrain the budgets $K_v$ to satisfy $0 \le K_v \le |\overrightarrow{T}_v|$ (clearly no subtree $\overrightarrow{T}_v$ can violate more than $|\overrightarrow{T}_v|$ edges).
For the sake of readability, we do not include this constraint in the recursive formula above. A similar recursion can be obtained for $OPT(u,K|-1)$.

One can verify that if we can compute $OPT(u,K|s)$ for all non-root nodes and all values of $K \leq K_n,s \in \{-1,1\}$ then we can find the optimum of the problem of our whole tree. To achieve this, simply attach a degree one node $r'$ to the root of the tree, add a directed edge $(r',r)$ and set the label of the root to equal 1.  Then we simply solve for $OPT(r,K|1)$, where $S(r,r')=1$ as well as $OPT(r,K,1)$, where $S(r',r)$ is $-1$ and return the minimum of the the values.

For a leaf node $w$, the value of $OPT(w,K'|s)$ can be calculated as follows: it is $\min(cost[s_w=-1],cost[s_w=1])$, for $K' \geq 1$. If $K=0$, it is $cost[s']$ where $s'$ is the unique label not violating the constraint $s  \ne s' \cdot S(w,p(w))$

We now show how to calculate $OPT(u,K_u|s)$ for any vertex in the tree, assuming $OPT$ has already been calculated for its children. To do this, we try both values of $s_{u}$, and then condition on its value to optimize
\[
\min_{\sum_{j \in [1,k]} K_j=K-\ind\crl{s  \ne s_{p(u)} \cdot S(u,p(u)) }}\sum_{u\in\brk{1,k}}OPT(j,K_j|s).
\]
The function $\sum_{v \in N_u}OPT(v,K_v|s)$ can be minimized using another layer of dynamic programming as follows: For $r \le s$, let $[r,s]$ be the set of integers between $r$ and $s$. Assuming we enumerate the vertices in $N(u)$ by $1,...,k:=|N(u)|$ and setting $K_j$ to be the budget for the $j$th node, we have the equality
$$\min_{\sum_{j \in [1,k]} K_j=K-\ind\crl{s  \ne s_{p(u)} \cdot S(u,p(u)) }}\sum_{u\in\brk{1,k}}OPT(j,K_j|s)$$
$$=\min_{K_1 \in [0,K-\ind\crl{s  \ne s_{p(u)} \cdot S(u,p(u)) }]}OPT(1,K_1|s)+\min_{\sum_{j \in [2,k]}K_j=K-K_1-\ind\crl{s  \ne s_{p(u)} \cdot S(u,p(u))}}\sum_{j\in\brk{2,k}}OPT(j,K_j|s).$$

The minimization problem  can be solved in time $O(|N(u)|K^2_n)$ time. We first calculate the minimum cost for the first two vertices where the number of constraints violated can range between $1$ to $K$. This can be done in time $O(K^2)$. We then examine the minimum cost for the first three vertices (assuming of course $u$ has at least three descendants) where the number of violated constraints ranges between $0$ and $K$. Since we have the information for the first two vertices, these values can be calculated again in time $O(K^2)$. We repeat this iteration until all descendants of $u$ are considered. It follows that the overall running time of this algorithm is $\sum_{u \in V}|N(u)|K^2_n=O(nK^2_n)$, since $T$ is a tree.

When $T$ is a path graph each node has a single child, the recursion collapses to time $O(nK_n)$.


\section{Further techniques for general graphs}

Here we give a simple proof that if the minimal degree of $G$ is $\Omega(\log n),$ then there is an algorithm that achieves arbitrarily small error for each vertex as $n\to\infty$ as soon as $q=1/2-\eps$ is constant.

\begin{theorem}\label{thm:large_degree}
There is an efficient algorithm that guarantees
\[
\En\brk*{\sum_{v\in{}v}\ind\crl*{\Yh_v\neq{}Y_v}}\leq{} \sum_{v\in{}V}\exp(-C\deg(v)\eps^{2}(1-2p)^{2}).
\]
for some $C>0$.
\end{theorem}
Observe that this rate quickly approaches $0$ with $n$ as soon as $\deg(G)=\Omega(\log{}n)$ (i.e., it has $o(n)$ Hamming error) . On the other hand, if degree is constant (say $d$), then even when $p=0$ the rate of this algorithm is only $e^{-dO(\eps^{2})}n$, so the algorithm does not have the desired property of having error approach $0$ as $p\to{}0$.
\begin{proof}[Proof of \pref{thm:large_degree}]
Fix a vertex $v$ and, for each vertex $u$ in its neighborhood, define an estimate $S_{u}=Z_{u}\cdot{}X_{uv}$. We can observe that $\Pr(S_u=Y_{v})= (1-p)(1-q) + pq = \frac{1}{2} + \eps{}(1-2p)$. Our algorithm will be to use the estimator $\Yh_{v} = \mathrm{Majority}(\crl*{S_u}_{u\in{}N(v)})$. Since each $S_u$ is independent, the Hoeffding bound gives that 
\[
\Pr(\Yh_v\neq{}Y_v)\leq{}\exp(-C\deg(v)\eps^{2}(1-2p)^{2}).
\]
Taking this prediction for each vertex gives an expected Hamming error bound of
\[
\En\brk*{\sum_{v\in{}v}\ind\crl*{\Yh_v\neq{}Y_v}}\leq{} \sum_{v\in{}V}\exp(-C\deg(v)\eps^{2}(1-2p)^{2}).
\]

\end{proof}

\end{document}